\documentclass[11pt]{article} 

\title{Learning Narrow One-Hidden-Layer ReLU Networks}
\author{
    Sitan Chen\thanks{Email: \texttt{sitanc@berkeley.edu}. Supported by NSF Award 2103300.} \\
    UC Berkeley
        \and 
    Zehao Dou\thanks{Email: \texttt{zehao.dou@yale.edu}} \\
    Yale
        \and
    Surbhi Goel\thanks{Email: \texttt{surbhig@cis.upenn.edu}} \\
    UPenn
        \and
    Adam Klivans\thanks{Email: \texttt{klivans@cs.utexas.edu}.  Supported by NSF award AF-1909204 and the NSF AI Institute for Foundations of Machine Learning (IFML).} \\
    UT Austin
        \and
    Raghu Meka\thanks{Email: \texttt{raghum@cs.ucla.edu}. Supported by NSF Tripods Institute grant 2217033.} \\
    UCLA
}

\usepackage{amsfonts, bbm}

\usepackage{imports}
\usepackage[margin=1in]{geometry}

\newcommand{\mf}{\mathfrak}

\newcommand{\He}{\mathrm{He}}
\renewcommand{\epsilon}{\varepsilon}

\newcommand{\fres}{f_{\sf res}}

\newcommand{\Sdet}{S_{\sf rem}}
\newcommand{\paramdist}{d_{\sf param}}
\newcommand{\ulam}{\Lambda}
\newcommand{\kres}{{k_{\sf res}}}
\newcommand{\Sfar}{S^{\sf far}}
\newcommand{\Sclose}{S^{\sf close}}
\newcommand{\radius}{\mathcal{R}}

\newcommand{\flamu}{f_{\vlam,\vec{u}}}
\newcommand{\fwlamu}{f_{w,\vlam,\vec{u}}}
\newcommand{\vlam}{\vec{\lambda}}
\newcommand{\quasipoly}{{\sf quasipoly}}
\newcommand{\vhlam}{\vec{\wh{\lambda}}}
\newcommand{\vhu}{\vec{\wh{u}}}
\newcommand{\whlam}{\wh{\lambda}}
\newcommand{\whu}{\wh{u}}
\newcommand{\Sundet}{\Gamma_{\sf undet}}
\newcommand{\lesim}{\lesssim}

\usepackage{mathtools}
\mathtoolsset{showonlyrefs}

\begin{document}

\maketitle
\begin{abstract}
We consider the well-studied problem of learning a linear combination of $k$ ReLU activations with respect to a Gaussian distribution on inputs in $d$ dimensions.  We give the first polynomial-time algorithm that succeeds whenever $k$ is a constant.  All prior polynomial-time learners require additional assumptions on the network, such as positive combining coefficients or the matrix of hidden weight vectors being well-conditioned.


Our approach is based on analyzing random contractions of higher-order moment tensors.  We use a multi-scale analysis to argue that sufficiently close neurons can be collapsed together, sidestepping the conditioning issues present in prior work. 
This allows us to design an iterative procedure to discover individual neurons. 
\end{abstract}

\tableofcontents

\section{Introduction}

We study the problem of PAC learning one-hidden-layer ReLU networks from labeled examples. In particular, we consider ReLU networks with $k$ neurons:
\begin{equation}
    \flamu(x) = \sum^k_{i=1} \lambda_i \cdot  \relu(\iprod{u_i,x})
\end{equation}
where $\relu(a) = \max(0,a)$ is the ReLU activation, and $u_1,\ldots,u_k\in\S^{d-1}$. Given examples of the form $(x, y) \in \mathbb{R}^d \times \mathbb{R}$ where $x$ is drawn from a distribution $\mathcal{D}_\mathcal{X}$ and $y = f(x)$ for some $f \in \mathcal{F}$, our goal is to learn a function $\wh{f}:\mathbb{R}^d\rightarrow \mathbb{R}$ with small test error, that is, $\mathbb{E}[(\wh{f}(x) - y)^2] \le \epsilon$.

This problem has inspired a large body of research in the machine learning community and acts as a benchmark for the design and analysis of novel learners. 
 The goal is to design provably efficient (sample complexity and running time being polynomial in the problem parameters $d, k, 1/\epsilon$) algorithms to PAC learn this class of functions under minimal assumptions.
The most common assumption is that $\mathcal{D}_\mathcal{X}$ is the standard Gaussian distribution. 
Even under this assumption, no known algorithm achieves runtime and sample complexity $\poly(d,1/\epsilon)$ even for constant $k$.


This paper presents the first algorithm for PAC learning one-hidden-layer ReLU networks under Gaussian inputs that succeeds whenever $k$ is constant:
\begin{theorem}\label{thm:main} Let $\mathcal{D}$ be the distribution over pairs $(x, y) \in \mathbb{R}^d \times \mathbb{R}$ where $x\sim\mathcal{N}(0, \Id)$ and $y = \flamu(x)$ for some $\vlam = (\lambda_1,\ldots,\lambda_k)$ and $\vec{u} = (u_1,\ldots,u_k)$. There is an algorithm that, given sample access to $\mathcal{D}$, has runtime and sample complexity $(d/\epsilon)^{h(k)}\cdot \log(1/\delta)$ for $h(k) = k^{O(\log^2 k)}$ and outputs a function $\wh{f}$ such that $\mathbb{E}[(\wh{f}(x) - y)^2] \le \epsilon$ with probability $1 - \delta$.
\end{theorem}

\noindent As we explain in Remark~\ref{remark:label_noise}, it is straightforward to verify that our algorithm also holds in the presence of unbiased i.i.d. noise on the labels $y$, so we omit these details for simplicity.

\subsection{Technical overview}


\paragraph{Tensors without a separation condition.} The starting point for our approach is the standard fact that we can obtain estimates of \emph{moment tensors} which encode high-order information about the unknown weight vectors $u_i$, namely tensors of the form $T_\ell \triangleq \sum_{i=1}^k \lambda_i u_i^{\otimes \ell}$. Indeed, given Gaussian examples $(x,\flamu(x))$, we can estimate the expectation of $y\cdot S_\ell(x)$, where the $\ell$-th order tensor-valued function $S_\ell$ is the \emph{$\ell$-th Hermite tensor},\footnote{See Appendix~\ref{app:hermite} for relevant background on Hermite polynomials.} and recover approximate tensors $\wh{T}_\ell$ such that $\|T_\ell - \wh{T}_\ell\|_F \leq \delta$ in time roughly $\ell^{O(\ell)} d^{2 \ell}/\delta^2$.

At this juncture, many existing works in this literature (see Related Work) try to apply tensor decomposition on $T_\ell$ to recover the $u_i$'s. Unfortunately, tensor decomposition is insufficient for us as it requires the weight vectors to be ``non-degenerate'' in some sense. This holds, for instance, if we assume $u_i$'s are \emph{well-separated}, meaning we have a non-negligible lower bound on $\|u_i - u_j\|_2$ for all $i,j \in [k]$ \cite{ma2016polynomial}. Unfortunately, directly applying tensor decomposition will fail in the absence of such separation assumptions.

\paragraph{Clumping.} To motivate our workaround, consider the simplest possible obstruction to the above approach: there exists a pair of indices $i\neq j$ such that $\|u_i - u_j\|_2$ is very small. The condition number of the weight vectors gets worse as this distance decreases, but intuitively if $u_i, u_j$ are sufficiently close, we should be able to \emph{clump} them together, that is, approximate them by a single neuron without incurring too much error in our approximation. While this seems promising, there is a critical hurdle here. In particular, suppose we group the $k$ neurons into $m$ clumps so that any two neurons in the same clump are distance at most $\gamma$ from each other, and any two neurons in different clumps are distance at least $\gamma$ from each other.\footnote{The careful reader will note that actually we can only ensure that neurons within the same clump are $k\gamma$-close and neurons within different clumps are $\gamma$-far, e.g. if the neurons lie on a line, but the extra factor of $k$ isn't important to the present discussion.} For every $i\in[m]$, let $u'_i$ denote some representative vector from the clump, so that $\|u'_i - u'_j \| \geq \gamma$ for all $i\neq j$ as desired. We might hope to apply tensor decomposition to the tensor $T_\ell' = \sum_{j=1}^{m} \lambda_j (u_j')^{\otimes \ell}$ to recover the representative vectors and thus learn the original network $\flamu$. Unfortunately, there is a critical issue. We only have approximate access to $T'_\ell$, but given the separation guarantee of $\gamma$ on $u_i'$ vectors, we would need an approximation to $T_\ell'$ with error $\delta' \ll \delta$; however, clumping vectors with distance $\delta$ will introduce error $\gg \delta$ in our tensor estimation. This quantitative trade-off will always be against us. 

To get over the above, we have to introduce several new ideas. The core idea is to use a \emph{multi-scale analysis} to pick which vectors to clump together strategically. 

\paragraph{From tensors to random contractions.} We will learn these clumps separately in multiple stages, rather than in ``one shot'' using tensor decomposition. To that end, instead of working with tensors $T_\ell$, we will work with matrices 
$$M_\ell^g = \sum_{i=1}^k \lambda_i \iprod{u_i,g}^{\ell-2} u_i u_i^\top\,,$$
where $g \sim \mathcal{N}(0,I)$ is a random Gaussian vector. Given estimates for $T_\ell$, we can form these by taking suitable tensor contractions. In place of tensor decomposition on $T_\ell$, we will use PCA on $M_\ell^g$ for various $\ell$. One challenge is that because we make no assumptions on $\lambda_1,\ldots,\lambda_k$, e.g. we do not assume they are nonnegative as in some prior works \cite{diakonikolas2020algorithms,smallcovers,ge2018learning}, many of these $M_\ell^g$ could be identically zero for all $g\in\R^d$, in which case PCA on such matrices provides no information. In fact, \cite{diakonikolas2020algorithms} gave a construction for which this is the case for all $\ell \le O(k)$ (see also \cite{goel2020superpolynomial}). An important component of our analysis will be to argue that if we consider all $\ell$ up to a sufficiently large constant multiple of $k$, there actually is enough information across the different matrices $M_\ell^g$ to learn $\flamu$.

\paragraph{First attempt: a single-stage algorithm.}

Let $v_i = \iprod{u_i,g}$. It is not hard to see that, up to some $\poly(k,d)$ factors, $|v_i - v_j| \propto \|u_i - u_j\|_2$ with high probability, i.e. the amount of separation among the $u_i$'s is inhereted by their projections $v_i$.

We can then do a case analysis. If $\max_{i\in[k]} v_i - \min_{i\in [k]} v_i \le \epsilon'$ (for a suitable $\epsilon = \epsilon/\poly(d,k)$), then we can find an approximation to $f_{\vlam, \vec{u}}$ using just one neuron. On the other hand, suppose two of the $v_i$'s are $\epsilon'$ far-away. For $i \in [k]$ and \emph{scale} $\gamma > 0$,
    \begin{align}
        \Sfar_i(\gamma) &\triangleq \brc{j\in[k]: |v_j - v_i| \ge \gamma} \\
        \Sclose_i(\gamma) &\triangleq \brc{j\in[k]: |v_j - v_i| \le T(\gamma)}\,,
    \end{align}
where $T(\gamma) \approx (\gamma/d)^{O(k)}$ is a suitable parameter. For $\gamma \le \epsilon'$, $\Sfar_i(\gamma)$ will be nonempty, and $\Sclose_i(\gamma)$ will consist only of $v_j$ which are very close to $v_i$.

An easy case for us would be when every index $i$ is \emph{gapped} in the sense that $v_j$'s for $j\neq i$ are either very close to $v_i$ or very far from $v_i$, with nothing in between. Quantitatively, suppose there were a choice of $\gamma \le \epsilon'$ such that $[k] = \Sclose_i(\gamma)\sqcup \Sfar_i(\gamma)$ for all $i\in[k]$. Then we could form clumps of neurons so that within any clump, any two neurons are $k T(\gamma)$-close, and any neurons in different clumps are $\gamma$-far. $kT(\gamma)$ is far smaller than $\gamma$, so that a certain PCA-based algorithm that works in the well-separated case can also be used to solve this gapped case. Furthermore, one can show that there always exists $\gamma$ which is at least some value $\underline{\gamma}$ depending solely on the problem parameters (e.g. $d, k, \epsilon$ rather than the weight vectors themselves) for which we are in the gapped case. The issue is that the largest $\underline{\gamma}$ for which one can show this is of order $d^{-k^{\Theta(k)}}$, and this turns out to be tight \--- imagine $u_1,\ldots,u_k$ lie on a line, and their pairwise separations scale roughly as $\epsilon', T(\epsilon'), T(T(\epsilon')), \cdots$. Nevertheless, this strategy already gives a polynomial-time algorithm in the case of $k = O(1)$, with runtime $d^{k^{O(k)}}$. We give additional details for this approach in Section~\ref{sec:simple}.

\paragraph{Better $k$ dependence: a multi-stage algorithm.} The bulk of this work is centered around refining the above guarantee with a multi-stage algorithm and multi-scale analysis to get a better dependence on $k$. The general idea is that it is not necessary to have a \emph{single} scale under which every index $i$ is simultaneously gapped. If we just want to learn a particular neuron $i$, we show that it suffices for there to be a scale $\gamma$ under which $i$ is gapped, even if no other indices are gapped at that scale (see Section~\ref{sec:case2a}). The proof of this relies on a certain estimate for power sum symmetric polynomials that may be of independent interest (Lemma~\ref{lem:powersum}). The key point is that if we just want $\gamma$ under which at least one single neuron is gapped, it suffices to go down to scale of order $d^{-k^{\Theta(\log k)}}$, rather than $d^{-k^{\Theta(k)}}$, before such a $\gamma$ exists (Lemma~\ref{lem:gap_exists})

Our final algorithm then proceeds in stages. In each stage, either all of the remaining $v_i$'s are $\epsilon'$-close to each other, in which case we can approximate the network by a single neuron. Otherwise, we identify a set of indices $i$, each of which has a corresponding scale $\gamma$ under which it is gapped, and argue that the set of top $k$ principal components across all $M^g_\ell$ with $\ell \le O(k)$ spans a subspace containing $i$. By enumerating over this subspace, we can learn the gapped neurons and make progress. We then recurse on the residual network given by subtracting the contribution from the neurons we have learned.

There is one last subtlety: given an approximation to the residual network at any given step of this algorithm, if the approximation error is $\xi$, then it turns out this error gets blown up, roughly speaking, to $\xi^{1/\Theta(k^{\log k})}$ in the next step of the algorithm. As a result, in order for the approximation error to still be small after $T$ iterations, we must estimate the matrices $M^g_\ell$ to error scaling exponentially in $k^{T\log k}$. Naively one can only ensure that a single new neuron is learned in each stage of the algorithm, meaning $T$ could potentially be as large as $k$, in which case we obtain no improvement over the single-stage algorithm above. Instead, via a careful combinatorial argument (Section~\ref{sec:shortpath}), we show that it is possible to learn enough neurons in each stage of the algorithm that we terminate in $T \le O(\log k)$ stages (Lemma~\ref{lem:shortpath}), thus yielding the improved runtime of $d^{k^{\log^2 k}}$ claimed in Theorem~\ref{thm:main}.


\begin{algorithm}[t]
\DontPrintSemicolon
\caption{\textsc{MultiScaleLearn}($f$)}
\label{alg:main}
    Sample random unit vector $g\in\S^{d-1}$ \;
    $\vhlam\gets\emptyset$, $\vhu\gets\emptyset$\;
    \For{$t = 1,\ldots,k$}{
        \For{$\ell = 1, 2, 4, \ldots, 2k+2$}{
            Compute estimates $\wh{T}_\ell$ of $\E{(y - f_{\vhlam,\vhu}(x))\cdot S_\ell(x)}$ from samples\;
            Evaluate $\wh{M}_\ell \gets \wh{T}_\ell(g,\cdots,g,:,:)$\;
        
        }
        Form a candidate estimate $h$ for the residual $f - f_{\vhlam,\vhu}$ as a neural network of the form $\mu^+\relu(\iprod{u,\cdot}) + \mu^-\relu(\iprod{-u,\cdot})$ (see proof of Lemma~\ref{lem:twoneuron})\;
        
        Compute the top-$k$ singular subspaces of $\wh{M}_2,\ldots,\wh{M}_{2k+2}$ and let $V$ be the joint $O(k^2)$-dimensional span of these subspaces.\;

        Form nets over $V$ and over $[-\radius,\radius]$ of granularity roughly $\poly(d,1/\epsilon,\radius)^{-k^{\Omega(\log^2 k)}}$ and guess an integer $m\in\brc{1,\ldots,k}$ and elements $u'_1,\ldots,u'_m$ and $\lambda'_1,\ldots,u'_m$ from each of these nets.\;
        
        (Nondeterministically) either add $\mu^+, \mu^-$ and $u, -u$ to $\vhlam$ and $\vhu$, or add $\lambda'_1,\ldots,\lambda'_m$ and $u'_1,\ldots,u'_m$ to $\vhlam$ and $\vhu$, or break out of the loop.\;

        Estimate $\norm{f - f_{\vhlam,\vhu}}^2$ from samples. If this is small, terminate and \Return{$f_{\vhlam,\vhu}$}.\;
    }
    \Return{Fail}
\end{algorithm}

The above procedure is summarized in the pseudocode for our algorithm (see Algorithm~\ref{alg:main}). We present it as a nondeterministic algorithm, but there are only $(d/\epsilon)^{k^{O(\log^2 k)}}$ possible choices in each iteration of the loop, for a total of $(d/\epsilon)^{k^{O(\log^2 k)}}$ computation paths. To form our final estimator, we simply try each of these paths, and as our rigorous guarantees imply that one of these paths yields a valid estimator, we output the $f_{\vhlam,\vhu}$ which achieves the best empirical loss.

\subsection{Related work}
\paragraph{Algorithms for PAC learning neural networks.} 
The design and analysis of algorithms for PAC learning various classes of simple neural networks has been very active in the last several years and has led to many innovative algorithms. These works make assumptions on the distribution of the inputs, the noise in the label, and the structure of the neural network to sidestep a large body of computational hardness results \cite{shalev2017failures,manurangsi2018computational,shamir2018distribution,vempala2019gradient,daniely2020hardness,goel2020superpolynomial,diakonikolas2020algorithms}. 

Examples of algorithmic techniques involved include tensor decomposition \cite{janzamin2015beating,sedghi2016provable,bakshi2019learning,ge2018learning,ge2018learning2,sewoong,diakonikolas2020algorithms}, kernel methods \cite{zhang2016l1,goel2017reliably,Daniely17, goel2019learning}, trajectory analyses of gradient-based methods \cite{zhong2017recovery,LiY17,vempala2019gradient,zgu,soltanolkotabi2017learning,zhangps17,diakonikolas2020approximation,LiMZ20,convotron,azll}, filtered PCA \cite{chen2022learning}, and  explicit covers for algebraic varieties \cite{smallcovers}.

Despite this flurry of work, the complexity of learning one-hidden-layer ReLU networks with respect to Gaussians remains open.  As mentioned above, the only results that achieve runtime and sample complexity polynomial in $d, k$ and $1/\epsilon$ require additional assumptions on the structure of the network, in particular (i) the matrix constructed from the weight parameters in the network is well-conditioned and/or (ii) the output layer weights are all positive. Under assumption (i), parameter recovery becomes possible, which is sufficient but unnecessary for PAC learning. The only known results that do not require a condition number bound (and hence do not recover parameters) are by \cite{diakonikolas2020algorithms} and \cite{chen2022learning}. The former requires assumption (ii), while the latter pays an exponential dependence on $1/\epsilon$ even for constant $k$. Our result removes assumptions (i) and (ii) and gets a polynomial dependence in the error parameter for constant $k$.

\paragraph{Statistical query algorithms.} Recent results by \cite{goel2020superpolynomial} and \cite{diakonikolas2020algorithms} rule out a $d^{o(k)}$ time algorithm for PAC learning one-hidden-layer ReLU networks for a large class of algorithms (including gradient descent on square loss) known as \textit{correlational statistical query} (CSQ) algorithms. A CSQ algorithm is allowed to access the data only via noisy correlational queries of the form $\mathbb{E}[y f(x)]$ for any query $f$ chosen by the learner. Our algorithm fits into this paradigm of CSQ algorithms and hence suffers from a $d^{\Omega(k)}$ runtime. Before our result, no known CSQ algorithm achieved $d^{r(k)}$ for any function $r$, which is independent of $d$ without additional assumptions on the structure of the network. We note that the recent result by \cite{chen2022learning} that achieves a polynomial dependence on $d$ for general networks is not a CSQ algorithm. 


\section{Technical Preliminaries}

\paragraph{Notation.} Given $f\in L^2(\R^d,\omega_d)$, where $\omega_d$ is the standard Gaussian measure on $\R^d$, let $\norm{f}_2$ denote its $L_2$ norm, that is, $\norm{f}^2_2 = \E[x\sim\calN(0,\Id)]{f(x)^2}$.

\subsection{ReLU networks and moment tensors}

It will be convenient to express one-hidden-layer ReLU networks in the following form, as the sum of a linear function with a linear combination of absolute values:

\begin{lemma}\label{lem:absvalue}
    Given a one-hidden-layer ReLU network $f(x) = \sum^k_{i=1}\mu_i \cdot \relu(\iprod{u_i,x})$, there exist $w\in\R^d$ and $\lambda_1,\ldots,\lambda_k\in\R$ such that
    \begin{equation}
        f(x) = \iprod{w,x} + \sum^k_{i=1} \lambda_i \cdot |\iprod{u_i,x}|
    \end{equation}
    for all $x\in\R^d$. Furthermore, $\norm{w} \le \sum_i |\lambda_i|$.
\end{lemma}

\begin{proof}
    Note that $\relu(z) = |z|/2 + z/2$ for any $z\in\R$, so we can write
    \begin{equation}
        f(x) = \sum^k_{i=1} \frac{\mu_i}{2} \cdot |\iprod{u_i,x}| + \Bigl\langle\sum^k_{i=1} \frac{\mu_i}{2} \cdot u_i, x\Bigr\rangle\,.
    \end{equation}
    We can thus take $w \triangleq \sum_i \frac{\mu_i}{2} \cdot u_i$ and $\lambda_i \triangleq \mu_i / 2$. The last part of the lemma is immediate.
\end{proof}

\noindent In light of Lemma~\ref{lem:absvalue}, given $w\in\R^d$ and $(\lambda_1,u_1),\ldots,(\lambda_k,u_k) \in \R\times\S^{d-1}$, let
\begin{equation}
    \fwlamu(x) \triangleq \iprod{w,x} + \sum^k_{i=1} \lambda_i \cdot |\iprod{u_i,x}|\,.
\end{equation}
When the $\lambda_i$, $u_i$, and $k$ are clear from context, we denote $\brc{\lambda_i, u_i}_{i\in[k]}$ by $(\vlam, \vec{u})$. Given $S\subseteq[k]$, we denote $\brc{\lambda_i, u_i}_{i\in S}$ by $(\vlam_S, \vu_S)$.
We call $k$ the \emph{width} of the network $\fwlamu$. Our bounds will depend on the $L_1$ norm of $\vlam$. Thus, henceforth suppose $\norm{\vlam}_1\le \radius$ for some parameter $\radius \ge 1$. Note that by the last part of Lemma~\ref{lem:absvalue}, we have
\begin{equation}
    \norm{w} \le \norm{\vlam}_1 \le \radius\,.
\end{equation}

Given $(\lambda_1,u_1),\ldots,(\lambda_k,u_k)\in\R\times\S^{d-1}$,  $g\in\S^{d-1}$, and $\ell\in\mathbb{N}$, define
\begin{equation}
    T_\ell(\brc{\lambda_i, u_i}_{i\in[k]}) \triangleq \sum_i \lambda_i u_i^{\otimes \ell} \qquad \text{and} \qquad M^g_\ell(\brc{\lambda_i, u_i}_{i\in[k]}) \triangleq \sum_i \lambda_i \iprod{u_i, g}^{\ell - 2} \cdot u_i u_i^\top\,,
\end{equation}
noting that $M^g_\ell$ can be obtained by contracting the \emph{moment tensor} $T_\ell$ along the direction $g$ for all of the first $\ell - 2$ modes.
When $g$ is clear from context, we denote $M^g_\ell$ by $M_\ell$.


Given $(\vlam,\vec{u}), (\vec{\lambda'},\vec{u'})\in (\R\times \S^{d-1})^k$, define the \emph{parameter distance}
\begin{equation}
    \paramdist((\vlam,\vec{u}), (\vec{\lambda'},\vec{u'})) \triangleq \min_{\pi}\max_i\brc{|\lambda_i - \lambda'_{\pi(i)}| + \norm{u_i - u'_{\pi(i)}}}\,,
\end{equation}
where the minimization is over all possible permutations of $k$ elements.


\subsection{Anti-concentration}
\label{sec:anti}

A main component of our algorithm is to contract estimates for the moment tensors $T_\ell$ along a random unit direction $g$ to get the matrices $M^g_\ell$ defined in the previous section. The most important feature of this operation is that it roughly preserves distances in the sense that if two weight vectors $u_i, u_j$ are close/far, their projections $\iprod{u_i, g}$ and $\iprod{u_j, g}$ are as well. Formally, we have the following elementary bounds, which follow by standard anti-concentration (see Appendix~\ref{app:anti}).

\begin{lemma}\label{lem:anti}
    With probability at least $4/5$ over random $g\in\S^{d-1}$, for all $i, j$ and $\sigma \in \{\pm 1\}$,
    \[
       \frac{c}{\sqrt{d}}\cdot \frac{1}{k^2} \le \frac{|\iprod{u_i + \sigma u_j, g}|}{\norm{u_i + \sigma u_j}} \le \frac{c'}{\sqrt{d}}\cdot \sqrt{\log k}
    \]
    for some absolute constants $c, c' > 0$. 
\end{lemma}

\begin{lemma}\label{lem:vlbd}
    With probability at least $9/10$ over random $g \in \S^{d-1}$, we have that $|\iprod{u_i,g}| \ge c/(k\sqrt{d})$ for all $i$, for some absolute constant $c > 0$.
\end{lemma}

\noindent Henceforth, we will condition on the event that $g$ satisfies Lemmas~\ref{lem:anti} and \ref{lem:vlbd}. We will denote
\begin{equation}
    v_i \triangleq \iprod{u_i,g} \label{eq:vdef}
\end{equation}
and, because of the absolute values in the definition of $\fwlamu$, we may assume without loss of generality that 
\begin{equation}
    0 \le v_1 \le \cdots \le v_k\,.
\end{equation}

\section{First Attempt: Learning in Time \texorpdfstring{$d^{k^{O(k)}}$}{dkk}}
\label{sec:simple}

Let us recall the technique used by \cite{diakonikolas2020algorithms} to PAC learn one-hidden-layer ReLU networks with non-negative combining weights $\lambda$. The approach is to approximately recover the subspace $U$ spanned by $\vec{u}$ using the matrix of degree-2 Chow parameters (or the second moment matrix), that is, $\mathbb{E} [y \cdot S_2(x)]= \sum_{i=1}^k \lambda_i u_iu_i^\top$. From here, one can see that the span of this second moment matrix can exactly recover the subspace $U$ spanned by $\vec{u}$. Using the non-negativity of $\lambda$, \cite{diakonikolas2020algorithms} argue that the span of the eigenvectors corresponding to the top-$k$ eigenvalues of an approximate second moment matrix (computed using samples) contains $k$-vectors $\hat{\vec{u}}$ such that $\|\flamu -f_{\vlam, \hat{\vec{u}}}\|_2$ is small. In fact, they show a stronger property: for all $i \in [k]$, $ \lambda_i \|u_i - \hat{u}_i\|^2$ is small. Following this, a brute force strategy on this subspace recovers an approximately-good hypothesis. 

For general, possibly negative, combining weights $\vlam$, if we can design a matrix $M \approx \sum_{i=1}^k |\lambda_i| u_iu_i^\top$ which we can estimate using samples, then we can use the technique from \cite{diakonikolas2020algorithms} to guarantee recovery of a $k$-dimensional subspace such that for all $i \in [k]$, $ |\lambda_i| \|u_i - v_i\|^2$ is small, which will guarantee small loss. Our first idea is to take $M$ to be a suitable linear combination of the moment tensor contractions $M_\ell^g = \sum^k_{i=1} \lambda_i \iprod{u_i,g}^{\ell - 2} \cdot u_i u_i^\top$. That is, we would like to find coefficients $\brc{\alpha_s}_{s\in[k]}$ such that 
\[
\sum_{s=1}^{k} \alpha_{s}M_{2s}^g \approx \sum_{i=1}^k |\lambda_i| u_iu_i^\top\,.
\] 
As long as the entries of $\alpha$ are not too large in magnitude, we can use a net-argument to brute-force over the choices of $\alpha$ and run \cite{diakonikolas2020algorithms} for each choice of $\alpha$. Showing that there exists such $\alpha_{s}$ for $s \in [k]$ reduces to showing that for all $i \in [k]$
\[
\sum_{s=1}^{k} \alpha_{s} v_i^{2(s-1)} = \sgn(\lambda_i).
\]
This is equivalent to showing the existence of a univariate polynomial $p$ of degree $k - 1$ with bounded norm that matches the sign pattern of $\lambda$ on inputs $\{v_1^2, \ldots, v_k^2\}$. 

If we had that $|v_i^2 - v_j^2|$ was lower bounded for all $i\neq j$, then using the following condition number bound for the $k \times k$ Vandermonde matrix generated by $v_1^2, \ldots, v_k^2$ would give us the desired $\alpha$ with bounded norm dependent on the gap:
\begin{lemma}[E.g., Fact 5.10 from \cite{chen2022learningpoly}]\label{lem:vandermonde}
    Given an $m\times m$ Vandermonde matrix $V$ generated by nodes $a_1,\ldots,a_m$ for which $|a_i - a_j| \ge \Delta$ for all $i\neq j$, and given $c\in\R^m$, there exists $\alpha$ for which $V\alpha = c$ such that $\norm{\alpha} \le m(1/\Delta)^{2m - 2}\norm{c}$.
\end{lemma}
\noindent Unfortunately, as we make no assumptions on the hidden weight vectors, it is not necessarily the case that $v^2_1,\ldots,v^2_k$ are well-separated. Nevertheless, we could try clumping together very close $v_i^2$'s into a single representative node (and adding their corresponding combining coefficients) so that each clump is well-separated while the approximation error from clumping does not blow up. In order to formally define clumping at scale $\gamma$, consider a graph on the indices with an edge between indices $i, j$ if $|v_i^2 - v_j^2| \le \gamma$. Then a clumping is specified by a set of disjoint connected subgraphs in this graph. The main challenge is that if we clump things together  at scale $\gamma$, then when we use the above result to construct $\brc{\alpha_s}$, our coefficients are of magnitude $O(1/\gamma^k)$. This would lead to an $O(1/\gamma^k)$ blow up in the error for the indices within each clump when we approximate them by the representative node for that clump. 

If we can find a scale $\gamma$ such that the elements in any clump are $\approx \gamma^{k}$-close while the clumps are $\gamma$-separated from each other, then this blow up does not hurt us. It turns out that we can always find a scale $\gamma = \epsilon^{k^c}$ for some $c \in [0, k-1]$ that satisfies this $\gamma$ versus $\gamma^k$ ``gap.'' To prove this, suppose we are at a scale $\epsilon^{k^c}$ and this property is not satisfied. Then there must be two clumps that are separated by $< \epsilon^{k^c}$, therefore, when we go up a scale $\epsilon^{k^{c-1}}$ then these two clumps would be combined together. This implies that whenever our condition is not satisfied, going up a scale reduces the number of clumps by 1. However, there can be at most $k$ clumps at the smallest scale. Thus at some scale, we must either have our desired gap or we can clump everything together. If we can clump everything together, this implies that the original network is well-approximated by at most two neurons, and we can learn these neurons directly.

This attempt would give us a runtime of $(d/\epsilon)^{k^{O(k)}}$. This argument can be formalized, however we only present the high-level intuition since our main result improves over this significantly. At a high level, the improvement is as follows. So far, we have given an approach that tries to learn the network in ``one shot'' by looking for a single scale at which there is a gap for all clumps simultaneously. Instead, in our improved algorithm, we learn the network over multiple steps. In each step, we identify disjoint clumps at several different scales such that each clump has a gap \emph{for its corresponding scale}, and use PCA to learn the neurons within these clumps. It turns out that instead of going down to scales as small as $\epsilon^{k^k}$, now it suffices to go down to scale $\epsilon^{k^{\log k}}$ (Lemma~\ref{lem:gap_exists}). We then prove that we can find enough such clumps at each step that after $O(\log k)$ steps, we can approximate the entire network, thus yielding an improved runtime of $(d/\epsilon)^{k^{\log^2 k}}$.

\section{Recursively Learning in Time \texorpdfstring{$d^{k^{O(\log^2 k)}}$}{dklog2k}}
\label{sec:main}

As the algorithmic guarantee in Theorem~\ref{thm:main} only beats brute force for $k \ll \poly(d)$, throughout we will assume this to simplify some estimates.

For some $S \subseteq[k]$, suppose we have access to $(\wh{\lambda}_i, \wh{u}_i)_{i\in S}$ for which \begin{equation}
    \norm{\fwlamu - f_{0,\vhlam,\vhu} - f_{w,\vlam_{S^c}, \vu_{S^c}}} \le \xi \,.\label{eq:l2close}
\end{equation}
In other words, $\fwlamu - f_{0,\vhlam,\vhu}$, i.e. the difference between the ground truth and what we have learned so far, is close to $\iprod{w,x}$ plus the subnetwork of $f$ indexed by $S^c$. We will refer to this latter network as
\begin{equation}
    \fres \triangleq f_{w,\vlam_{S^c}, \vu_{S^c}}\,.
\end{equation}

For notational convenience, we assume without loss of generality that $S = \brc{\kres + 1,\ldots, k}$, where $\kres \triangleq k - |S|$.
Define the parameters
\begin{equation}
    \epsilon' \triangleq \epsilon / \poly(d,k,\radius) \qquad \ulam \triangleq \poly(d\radius/\epsilon)\cdot \xi^{1/k^{\Theta(\log k)}} \qquad \underline{\gamma} \triangleq \Bigl(\frac{\ulam \epsilon}{\radius d}\Bigr)^{k^{\Theta(\log k)}} \,. \label{eq:paramdef}
\end{equation} 
The parameter $\xi$ will be sufficiently small that $\Lambda \ll 1$.



\noindent One important case in which we will show it is possible to learn a neuron $i\in[\kres]$ is when there is a significant \emph{gap} among the distances from other $v_j$ to $v_i$, i.e. every $v_j$ is either very far from $v_i$ or very close to $v_i$. Formally, given $i\in[\kres]$ and $\gamma > 0$, define
\begin{align}
    \Sfar_i(\gamma) &\triangleq \brc{j\in[\kres]: |v_j - v_i| \ge \gamma} \\
    \Sclose_i(\gamma) &\triangleq \brc{j\in[\kres]: |v_j - v_i| \le T(\gamma)}\,,
\end{align}
where
\begin{equation}
    T(\gamma) \triangleq   \frac{\ulam^{10}}{\mathcal{R}^2} \, (\gamma/d)^{\Theta(k)}\,. \label{eq:Tdef}
\end{equation}

\begin{definition}
    Given $i\in[\kres]$ and $0 < \gamma < 1$, we say that $\gamma$ is a \emph{gapped scale for $i$} if $[\kres] = \Sclose_i(\gamma)\sqcup \Sfar_i(\gamma)$ and $\gamma \ge \underline{\gamma}$. Further, we say that $i$ is \emph{detectable at scale $\gamma$} if $|\sum_{j\in\Sclose_i(\gamma)} \lambda_i| > \ulam$.
\end{definition}

\begin{figure}[h]
    \centering
    \includegraphics[width=0.9\textwidth]{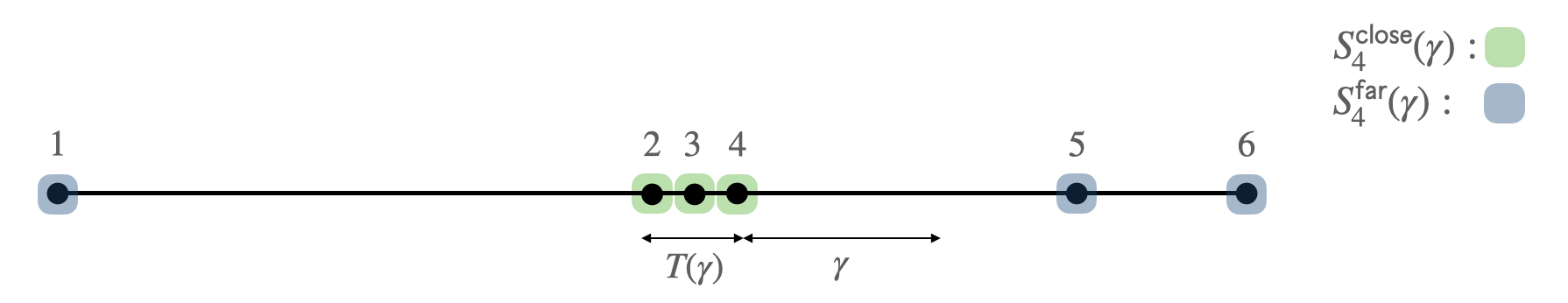}
    \caption{Illustration of $\Sclose_i(\gamma)$ and $\Sfar_i(\gamma)$. The figure shows that $\gamma$ is a gapped scale for $i = 4$, but note that because points 5 and 6 are distance strictly less than $\gamma$ apart and strictly greater than $T(\gamma)$ apart, $\gamma$ is not a gapped scale for $i = 5,6$.}
    \label{fig:gap}
\end{figure}

In this section we will show that one of two things can happen which will allow us to learn part of or all of $\fres$. 
\begin{enumerate}[leftmargin=*]
    \item If $\max_{i\in[\kres]} v_i - \min_{i\in [\kres]} v_i \le \epsilon'$ (this includes the case where $\kres = 0$), then we can find an approximation to $\fres$ as a linear combination of two ReLUs and terminate.
    \item There exists at least one pair $(i,\gamma)\in[\kres]\times[\underline{\gamma},1)$ such that $\gamma$ is a gapped scale for $i$. In this case, let $\calJ$ be any set of such pairs $(i,\gamma)$ such that the sets $\Sclose_i(\gamma)$ are all disjoint (we will specify the precise $\calJ$ that we work with later).
    We will show that:
    \begin{enumerate}
        \item For those $(i,\gamma) \in\calJ$ for which $i$ is detectable at scale $\gamma$, we can use the different $M_\ell$ to construct a net containing a vector close to $u_j$ for some $j\in \Sclose_i(\gamma)$.
        \item For all other $(i,\gamma) \in \calJ$, the subnetwork $\fres'$ of $\fres$ given by the corresponding neurons in $\fres$ is well-approximated by the zero function.
    \end{enumerate}
    Suppose we know which $(i,\gamma)\in\calJ$ fall into Case 2a versus Case 2b (recall that our algorithm is nondeterministic, so along some computation path we will have correctly guessed these). Then after brute-forcing over weights $\wh{\lambda}$ to assign to the neurons learned in Case 2a, we can update the set of pairs $(\whlam_i, \whu_i)_{i\in S}$ to a set of pairs $(\wt{\lambda}_i, \wt{u}_i)_{i\in S'}$ for some $S'\supsetneq S$ such that \begin{equation}
        \norm{\fwlamu - f_{0,\wt{\vlam},\wt{\vu}} - f_{w,\vlam_{{S'}^c}, \vu_{{S'}^c}}} \le \xi^*
    \end{equation}
    for some new error $\xi^*$. We can then recurse, until the set $S'$ is all of $[k]$ and we have learned an approximation to the entire original network $\fwlamu$.
\end{enumerate}

\noindent \emph{A priori}, one might need up to $k$ recursive steps to learn $\fwlamu$, but across this many steps the initial error $\xi$ would be blown up by a factor which is doubly exponential in $k$. We will instead argue that starting with $S = \emptyset$, there is sequence of choices of $\calJ$ across only $O(\log k)$ recursive steps such that we end up terminating with a sufficiently good approximation to the original network $\fwlamu$. This will require a delicate combinatorial argument (Section~\ref{sec:shortpath}) which is crucial to obtaining our $d^{\quasipoly(k)}$ runtime and sample complexity.



\subsection{Moment tensor estimations}

Here we show how to estimate $T_\ell(\vlam_{[\kres]}, \vu_{[\kres]})$ using the parameters $(\vhlam, \vhu)$ as well as Gaussian samples $\brc{(x_a,y_a)}_{a\in[N]}$ labeled by the original function $\fwlamu$. We use the following standard guarantee for approximately recovering moment tensors from samples, whose proof is deferred to Appendix~\ref{app:momentest_defer}.

\begin{lemma}[Moment estimation]\label{lem:moment_est}
    Let $\ell \in \brc{1,2,4,6,\ldots}$. Let 
    \begin{equation}
        c_\ell \triangleq \begin{cases}
            1/2 & \text{if} \ \ell = 1 \\
            \frac{\He_\ell(0) + \ell H_{\ell-2}(0)}{\sqrt{2\pi\cdot \ell!}} & \text{if} \ \ell \ \text{even}
        \end{cases}
    \end{equation}
    Given samples $\brc{(x_i, \fwlamu(x_i)}_{i=1,\ldots,N}$ for $x_i\sim\calN(0,\Id)$ and $N\ge \ell^{O(\ell)} d^{2\ell} \radius^2 / \xi^2$, the tensor $\wh{T} = \frac{1}{2c_\ell N}\sum^k_{i=1} \fwlamu(x_i) \cdot S_{\ell}(x_i)$, where $S_{\ell}$ is the $\ell$-th normalized probabilist's Hermite tensor, satisfies $\norm{\wh{T} - T_\ell(\vlam,\vu)}_F \le \xi$ if $\ell$ is even, and otherwise satisfies $\norm{\wh{T} - w}_2 \le \xi$ if $\ell = 1$.
\end{lemma}


\noindent The following lemma, whose proof is deferred to Appendix~\ref{app:residual_defer}, shows that we can approximate the Hermite coefficients of $\fres$ by empirically estimating the Hermite coefficients of $\fwlamu - f_{0,\wh{\vlam},\wh{\vu}}$:

\begin{lemma}\label{lem:estimate_residual_moments}
    Let \begin{equation}
        \xi' \triangleq \poly(d) \cdot (\radius + \norm{\vhlam}_1)\cdot \xi\,.
    \end{equation}
    Given the parameters $(\vhlam,\vhu)$ and Gaussian samples $(x_1,y_1),\ldots,(x_N,y_N)$ labeled by the original function $\fwlamu$, we have for all $\ell = 2,\ldots,2k+2$ that
    \begin{align}
        \left\|\frac{1}{N}\sum^N_{a=1} (y_a - f_{0,\vhlam, \vhu}(x_a)) \, x_a - w \right\| \le \xi' \\
        \left\|\frac{1}{2c_\ell N}\sum^N_{a=1} (y_a - f_{0,\vhlam, \vhu}(x_a)) \, S_\ell(x_a) - T_\ell(\vlam_{[\kres]}, \vu_{[\kres]}) \right\| \le \xi'
    \end{align}
    provided $N\ge d^{O(k)}\,(\radius^2 + \norm{\vhlam}^2_1)/\xi'^2$. 
\end{lemma}

\begin{remark}\label{remark:label_noise}
    Note that if instead of getting samples of the form $(x,\fwlamu(x))$, we had samples of the form $(x,\fwlamu(x) + \zeta)$ where $\zeta$ is independent mean-zero noise, then the estimators for $T_\ell(\vlam,\vu)$ and $w$ defined above are still unbiased estimators for these quantities. Our algorithm uses its samples solely to form these estimators, so provided $\zeta$ has bounded second moment so that the variance of these estimators is not too large, our algorithm will still work in the presence of such label noise.
\end{remark}

\subsection{Case 1: two-neuron approximation}

\begin{lemma}\label{lem:twoneuron}
    Suppose $\max_{i\in[\kres]} v_i - \min_{i\in[\kres]} v_i \le \epsilon'$. Then there is an efficient algorithm that takes the parameters $(\vhlam, \vec{\wh{u}})$ as well as $\poly(d)(\radius^2 + \norm{\vhlam}^2_1)/\xi'^2$ samples from the original function $\fwlamu$ and outputs weights $\mu^+, \mu^-$ and a vector $u$ such that the network
    \begin{equation}
        h \triangleq \mu^+ \relu(\iprod{u,\cdot}) + \mu^- \relu(\iprod{-u,\cdot})
    \end{equation} satisfies $\norm{\fwlamu - f_{0,\vhlam,\vhu} - h} \lesim \poly(d,\radius) (\epsilon' + \xi')$.
\end{lemma}

\begin{proof}[Sketch, see Appendix~\ref{app:defer_twoneuron}]
    The condition on $v_1,\ldots,v_\kres$ and the fact that projection along $g$ roughly preserves distances among $u_1,\ldots,u_\kres$ implies that all of the weight vectors in $\fres$ are close in Euclidean distance. As a result, $\fres$ is well-approximated by a ReLU network which only depends on the projection of the input to a single direction, i.e. a width-2 network whose weight vectors are $u$ and $-u$ for some vector $u$. This can be done by considering suitable linear combinations of the first Hermite coefficients of $\fres$, namely $w$, together with a contraction of the second moment matrix, which is simply $\sum_{i\in[\kres]} \lambda_i v_i u_i$, both of which can be estimated from samples by Lemma~\ref{lem:estimate_residual_moments}.
\end{proof}






\subsection{Existence of a gapped scale}

As a warmup, here we show that if we are not in Case 1, there exists $i\in[\kres]$ for which there is a gapped scale $\gamma$.

\begin{lemma}\label{lem:gap_exists}
    Suppose $\max_{i\in[\kres]} v_i - \min_{i\in[\kres]} v_i > \epsilon'$. Then there exists at least one index $i\in[\kres]$ for which there is a gapped scale $\gamma$.
\end{lemma}

\begin{proof}
    In this proof, assume without loss of generality that $v_1 \le \cdots \le v_\kres$, and define $\delta_i \triangleq v_i - v_{i-1}$ for $1 < i \le \kres$. By averaging, there exists some $i_0$ for which $\delta_{i_0} \ge \epsilon'/\kres$. Suppose without loss of generality that $i_0 > \kres/2$. 
    
    If $\delta_j \le T(\epsilon'/\kres)/\kres$ for all $i_0 < j \le \kres$, then $\epsilon'/\kres$ is a gapped scale for $i_0$. Otherwise, there exists $i_0 < i_1 \le \kres$ for which $\delta_{i_1} > T(\epsilon'/\kres)/\kres$. Suppose without loss of generality that $i_1$ is closer to $\kres$ than $i_0$. 
    
    Again, if $\delta_j \le T(T(\epsilon'/\kres)/\kres)/\kres$ for all $i_1 < j \le \kres$, then $T(\epsilon'/\kres)/\kres$ is a gapped scale for $i_0$. Otherwise, we can continue this binary search procedure at most $O(\log k)$ times. 

    Finally, we verify that $\underline{\gamma}$ is smaller than the result of iterating $z\mapsto T(z/\kres)/\kres$ for $O(\log k)$ times starting from $\epsilon'$, which will prove that there exists a gapped scale $\gamma$ for some $i\in[\kres]$.

    To verify this, note that $T(z/\kres)/\kres \ge T(z/k)/k \ge (\ulam^{10}/\radius^2)(z/d)^{\Theta(k)}$. So iterating this $O(\log k)$ times starting from $z = \epsilon'$ gives a quantity which is at least $\Bigl(\frac{\ulam \epsilon}{\radius d}\Bigr)^{k^{\Theta(\log k)}} \approx \underline{\gamma}$.
\end{proof}

\subsection{Case 2a: detectable neurons}
\label{sec:case2a}

Here we show that for any $(i,\gamma)$ for which $\gamma$ is a gapped scale for $i$ and furthermore $i$ is detectable, we can use a certain PCA-based procedure to produce a net over vectors, at least one of which is close to $u_i$.

The main ingredient in the proof is the following consequence of an estimate for power sums (Lemma~\ref{lem:powersum}) that we prove in Section~\ref{sec:powersum}.

\begin{lemma}\label{lem:powersum_specific}
    Consider $i\in[\kres]$ and $\gamma$ which is a gapped scale for $i$ and such that $i$ is detectable at $\gamma$. Then for any projector $\Pi\in\R^{d\times d}$ and $r \triangleq \Pi u_i$, we have that
    \begin{equation}
        r^\top M_\ell(\vlam_{[\kres]}, \vec{u}_{[\kres]})\, r \ge C_1 \,\norm{r}^4_2 - C_2
    \end{equation}
    for
    \begin{equation}
        C_1 = \ulam\,(\gamma/d)^{\Theta(k)} \qquad C_2 \lesim \radius \poly(d)\, T(\gamma)^{1/2}\,.
    \end{equation}
    for some even integer $2 \le \ell \le 2k + 2$ and absolute constants $c,C > 0$. 
\end{lemma}

\begin{proof}
    Take $k$ and $k'$ in Lemma~\ref{lem:powersum} to be $\kres$ and $|\Sclose_i(\gamma)|$ respectively. Permute $[\kres]$ so that $i$ is now the first element, $\Sclose_i(\gamma)$ consists of the first $k'$ elements, and $\Sfar_i(\gamma)$ consists of the remaining. For every $i\in[\kres]$, let $q_i = \lambda_i \iprod{u_i,r}^2$, and let $v_i$ be as defined in \eqref{eq:vdef}. Note that
    \begin{equation}
        r^\top M_\ell(\vlam_{[\kres]}, \vec{u}_{[\kres]})\, r = \sum^{\kres}_{i=1} \lambda_i \iprod{u_i,g}^{\ell-2} \iprod{u_i,r}^2 = \sum^{\kres}_{i=1} q_i v_i^{\ell} = \iprod{v^{\odot\ell},q}\,. \label{eq:rMr}
    \end{equation}
    As we are conditioning throughout on the event of Lemma~\ref{lem:vlbd}, we have that $|v_i| \ge \alpha$ for all $i$ for $\alpha \triangleq c/(k\sqrt{d})$. We can take $R$ in Lemma~\ref{lem:powersum} to be the assumed bound $\radius$ on $\norm{\vlam}_1$, as $\norm{q}_\infty \le \norm{\vlam}_\infty \le \radius$. And by the hypothesis of the lemma, $|v_i - v_j| \le \beta$ for all $j\in\Sclose_i(\gamma)$ for $\beta \triangleq T(\gamma)$ and $|v_i - v_j| \ge \gamma$ for all $j\in\Sfar_i(\gamma)$.

    By Lemma~\ref{lem:powersum} and \eqref{eq:rMr}, we conclude that
    \begin{align}
        r^\top M_\ell(\vlam_{[\kres]}, \vec{u}_{[\kres]})\, r &\ge \frac{\tau}{2\kres}\Bigl(\frac{c^2\gamma^2}{4k^3d}\Bigr)^\kres - C\radius\kres\,(|\Sclose_i(\gamma)| - 1)\, T(\gamma) \\
        &\ge \frac{\tau}{2k}\Bigl(\frac{c^2\gamma^2}{4k^3 d}\Bigr)^k - C\radius k^2 T(\gamma)
    \end{align}
    for $\tau \triangleq |\sum_{j\in \Sclose_i(\gamma)} q_j|$. It remains to relate $\tau$ to $\norm{r}^4$. Recalling that $r = \Pi u_i$, we have
    \begin{equation}
        \iprod{u_i, r}^2 = \iprod{u_i, \Pi u_i}^2 = \norm{\Pi u_i}^4 = \norm{r}^4\,. \label{eq:useproj}
    \end{equation}
    In addition, for every $j\in\Sclose_i(\gamma)$, we have
    \begin{align}
        |q_j - \lambda_j \iprod{u_i,r}^2| &\le |\lambda_j| \cdot |\iprod{u_j, r}^2 - \iprod{u_i, r}^2| \\
        &\le 2|\lambda_j|\,\min_{\sigma\in\brc{\pm 1}} \norm{u_i - \sigma\cdot u_j} \\
        &\lesim |\lambda_j| \, k\sqrt{d}\cdot \min_{\sigma\in\brc{\pm 1}} \norm{v_i - \sigma\cdot v_j} \\
        &\le T(\gamma)^{1/2} |\lambda_j|\,k\sqrt{d}\,, \label{eq:wjdiff}
    \end{align}
    where the first and third steps are by Lemma~\ref{lem:squaretosigneddiff}, and the second step is by the fact that we are conditioning on the event of Lemma~\ref{lem:anti}.
    Combining \eqref{eq:useproj} and \eqref{eq:wjdiff}, we conclude that
    \begin{align}
        \tau &= \left|\sum_{j\in\Sclose_i(\gamma)} q_j\right| \ge \ulam\norm{r}^4 - T(\gamma)^{1/2} \radius\, k\sqrt{d}\,,
    \end{align}
    from which we get the following more quantitative version of the claimed bound.
    \begin{align}
        r^\top M_\ell(\vlam_{[\kres]}, \vec{u}_{[\kres]})\, r &\ge  \frac{\ulam}{2k}\Bigl(\frac{c^2\gamma^2}{4k^3 d}\Bigr)^k \cdot \norm{r}^4_2 - C\radius k^2 T(\gamma) - \frac{\sqrt{d}}{2}\Bigl(\frac{c^2\gamma^2}{4k^3 d}\Bigr)^k\, T(\gamma)^{1/2}\radius\,. \qedhere
    \end{align}
\end{proof}

\noindent To see why such a lower bound on the quadratic form is useful, we show next that it can be used to obtain a net over vectors containing one which is close to some $u_i$:

\begin{lemma}\label{lem:pca}
    Let $T\subseteq[k]$, and let $\wh{M}_\ell$ be approximations to $M_\ell(\vlam_T, \vu_T)$ satisfying
    \begin{equation}
        \norm{M_\ell(\vlam_T, \vu_T) -\wh{M}_\ell} \le \eta
    \end{equation}
    for all $\ell \in \{2, \ldots, 2k, 2k+2\}$. Let $V$ be the span of all the top-$k$ singular values of $\wh{M}_2, \wh{M}_4, \ldots, \wh{M}_{2k + 2}$. Now let $u'_i = \mathsf{proj}_V(u_i)$ for $i \in T$, and $r_i = u_i - u'_i$. 

    If for some $i\in T$ and some $\ell\in\{2,\ldots,2k + 2\}$, 
    \begin{equation}
        r_i^\top M_\ell(\vlam_T, \vu_T) r_i \ge C_1\, \iprod{r_i, u_i}^2 - C_2, \label{eq:pca_cond}
    \end{equation}
    then if we define $\calS$ to be an $\upsilon$-net over unit vectors in $V$ for $\upsilon \asymp (\eta/C_1)^{1/2} + (C_2/C_1)^{1/4}$, then $\calS$ contains a vector which is $2 \upsilon$-close to $u_i$.
\end{lemma}

\begin{proof}
   We will denote $M_\ell(\vlam_T, \vu_T)$ by $M_\ell$ for the rest of the lemma. First let us upper bound $r_i^\top M_\ell r_i$. We have
   \begin{equation}
    r_i^\top M_\ell r_i = r_i^\top \wh{M}_\ell r_i + r_i^\top (M_\ell - \wh{M}_\ell) r_i \le  \lambda_{k + 1} (\wh{M}) \norm{r_i}^2 + \norm{M_\ell - \wh{M}_\ell}_{\sf op}\norm{r_i}^2 \le 2\eta \norm{r_i}^2
   \end{equation}
   Here the first inequality follows from observing that $r_i$ is orthogonal to $V$ which contains the top-$k$ singular subspace of $\wh{M}_\ell$. The last inequality follows from the following two facts: (i) Weyl's inequality, $\norm{\lambda(M_\ell) - \lambda(\wh{M}_\ell)}_1 \le \eta$ (where $\lambda(A)$ are the eigenvalues of $A$ sorted in decreasing order), and (ii) $\lambda_{k+1}(M_\ell) = 0$ since ${\sf rank}(M_\ell) = k$.
   
   Since $\iprod{r_i, u_i} = \iprod{r_i, r_i + u_i'} = \norm{r_i}^2$, we have
    \begin{align*}
        &2\eta\norm{r_i}^2 \ge r_i^\top M_\ell r_i \ge C_1\norm{r_i}^4 - C_2.
    \end{align*}
    This gives us that $\norm{r_i}\le \sqrt{\frac{\eta + \sqrt{\eta^2 + C_1C_2}}{C_1}} \le \sqrt{\frac{2 \eta}{C_1}} + \sqrt[4]{\frac{C_2}{C_1}}$. Since $u'_i$ is the projection of $u$ to $V$, this implies that $V$ contains a point close to $u_i$ in $\ell_2$ distance. 

    Let $\calS$ be an $\upsilon$-net over unit vectors in $V$ for $\upsilon \asymp (\eta/C_1)^{1/2} + (C_2/C_1)^{1/4}$. Then by triangle inequality, we know that there exists a point $z \in N$ such that $\norm{z- u_i} \le \norm{z - u'_i} + \norm{r_i} \le 2 \upsilon$.
\end{proof}

\noindent We now apply Lemma~\ref{lem:pca} using the bound in Lemma~\ref{lem:powersum_specific}.

\begin{lemma} \label{lem:listS}
    Under the hypotheses of Lemma~\ref{lem:powersum_specific}, there is an algorithm that takes the parameters $(\vhlam, \vec{\wh{u}})$ as well as $\poly(d)\, (\radius^2 + \norm{\vhlam}^2_1)/\xi'^2$ samples from the original function $\fwlamu$ and outputs a list $\calS$ of unit vectors, containing vectors which are $\upsilon$-close to $u_i$ for every $i\in[\kres]$ for which there exists a gapped scale, where
    \begin{equation}
        \upsilon \asymp (\xi'/C_1)^{1/2} + (C_2 / C_1)^{1/4} \lesssim \ulam
        \,.\label{eq:newerror}
    \end{equation}
\end{lemma}

\begin{proof}
    Take any such $i\in[\kres]$ with corresponding gapped scale $\gamma$. In Lemma~\ref{lem:pca}, we will take $T \triangleq [\kres]$, $\wh{M}_\ell$ given by the empirical estimators from Lemma~\ref{lem:estimate_residual_moments}, $\eta \triangleq \xi'$, and $C_1, C_2$ as in Lemma~\ref{lem:powersum_specific}. Note that by our choice of $T(\gamma)$ in \eqref{eq:Tdef}, $C_2 / C_1$ is bounded by an arbitrarily small constant multiple of $\Lambda$ for all $\gamma > 0$, and by our choice of $\Gamma$, $(\xi'/C_1)^{1/2} \lesssim \ulam$. Then we conclude that for $\upsilon$ as in \eqref{eq:newerror},
    we can recover all $i$ for which there exists a gapped scale to error of order $\upsilon$. The sample complexity follows from the guarantee of Lemma~\ref{lem:estimate_residual_moments}.
\end{proof}

\subsection{Case 2b: ignoring undetectable neurons}

Here we show that if there are undetectable neurons, then we can approximate them by zero:

\begin{lemma}\label{lem:subtract_linear}
    Let $\Sundet\subset[\kres]\times[\underline{\gamma},1)$ be a set of pairs $(i,\gamma)$ such that $\gamma$ is a gapped scale for $i$ and furthermore $i$ is undetectable at scale $\gamma$. Additionally, suppose that the sets $\brc{\Sclose_{i}(\gamma)}_{(i,\gamma)\in \Sundet}$ are all disjoint.
    
    Then for $\Sdet\triangleq [\kres]\backslash \cup_{(i,\gamma)\in\Sundet} \Sclose_i(\gamma)$,
    \begin{equation}
        \norm{\fres - f_{w,\vec{\lambda}_{\Sdet}, \vu_{\Sdet}}} \lesim T(\gamma) k^3\sqrt{d} + k\ulam 
    \end{equation}
    and thus that
    \begin{equation}
        \norm{\fwlamu - f_{0,\vhlam,\vhu} - f_{w,\vec{\lambda}_{\Sdet}, \vu_{\Sdet}}}\lesim T(\gamma) k^3\sqrt{d} + k\ulam + \xi \lesim \ulam\,.
    \end{equation}
\end{lemma}

\begin{proof}
    Let $(i,\gamma) \in \Sundet$. By the fact that we are conditioning on the event of Lemma~\ref{lem:anti}, 
    \begin{equation}
        \norm{u_i - u_j} \lesim T(\gamma) k^2\sqrt{d}   \label{eq:unitsclose}
    \end{equation}
    We can thus apply Lemma~\ref{lem:paramtoL2} to the networks $\sum_{j\in \Sclose_i(\gamma)} \lambda_j |\iprod{u_j,\cdot}|$ and $(\sum_{j\in \Sclose_i(\gamma)} \lambda_j) |\iprod{u_i,\cdot}|$ to get
    \begin{equation}
        \Bigl\|\sum_{j\in\Sclose_i(\gamma)} \lambda_j |\iprod{u_j, \cdot}| - \Bigl(\sum_{j\in\Sclose_i(\gamma)} \lambda_j\Bigr)\cdot |\iprod{u_i,\cdot}|\Bigr\| \lesim k^2\sqrt{d} \radius T(\gamma)\cdot |\Sclose_i(\gamma)|\,.
    \end{equation}
    Additionally, because $|\sum_{j\in\Sclose_i(\gamma)} \lambda_j| \le \ulam$, 
    \begin{equation}
        \Bigl\|\Bigl(\sum_{j\in\Sclose_i(\gamma)} \lambda_j\Bigr)\cdot |\iprod{u_i,\cdot}|\Bigr\| \lesim \ulam\,.
    \end{equation}
    The first part of the lemma follows upon noting that $\sum_{(i,\gamma)\in \Sundet} |\Sclose_i(\gamma)| \le k$ and that, by definition of $\fres$ and disjointness of $\Sclose_i(\gamma)$ for $(i,\gamma)\in\Sundet$,
    \begin{equation}
        \fres - f_{w,\vec{\lambda}_{\Sdet}, \vu_{\Sdet}} = \sum_{(i,\gamma)\in\Sundet} \sum_{j\in\Sclose_i(\gamma)} \lambda_j |\iprod{u_j,\cdot}|\,.
    \end{equation}
    The second part then follows by \eqref{eq:l2close}.
\end{proof}




\subsection{Combining the cases}

We now put together the analysis from the preceding sections. We would like to show that starting from an approximating network $f_{\vhlam,\vhu}$ which satisfies \eqref{eq:l2close}, after one step of either Case 1 or 2, we end up with an $\epsilon$-close estimate of the original function $\flamu$, or otherwise we make progress by approximately learning some new neurons, corresponding to the set of pairs $(i,\gamma)$ given by $\calJ$, from the residual network.

First, following Lemma~\ref{lem:subtract_linear}, let $\Sundet\subseteq\calJ$ denote the set of pairs $(i,\gamma)$ for which $i$ is undetectable at scale $\gamma$, and recall the definition of $\Sdet = [\kres]\backslash \cup_{(i,\gamma)\in\Sundet} \Sclose_i(\gamma)$. By Lemma~\ref{lem:subtract_linear}, we can effectively ignore the neurons in $\cup_{(i,\gamma)\in\Sundet} \Sclose_i(\gamma)$. Among the neurons in $\Sdet$, we can use the analysis for Case 2a to recover those in $\calJ\backslash \Sundet$, i.e. those which have gapped scales at which they are detectable. For these, we can then brute-force over possible weights, resulting in an approximation to the subnetwork given by the neurons in $\calJ\backslash\Sundet$.

\begin{lemma} \label{lem:epscoverlambda}
    Let $\calJ\subseteq S^c$ denote any subset of pairs $(i,\gamma)$ for which $\Sclose_i(\gamma)$ are all disjoint and $\gamma$ is a gapped scale for $i$. Let $\Sundet\subseteq\calJ$ denote the set of $(i,\gamma)$ in $\calJ$ such that $i$ is undetectable at scale $\gamma$.
    
    Suppose for every $(i,\gamma)\in \calJ\backslash\Sundet$, we have a vector $\wt{u}_i$ satisfying $\norm{\wt{u}_i - u_i} \le \upsilon$. Then if we define $\mathcal{S}_\vlam$ to be an $\upsilon$-scale discretization of $[-\mathcal{R}, \mathcal{R}]$ and take $S^\star\triangleq [\kres] \backslash \cup_{(i,\gamma)\in\calJ} \Sclose_i(\gamma)$, then there exist $\brc{\wt{\lambda}_i}_{(i,\gamma)\in\calJ\backslash\Sundet}$ taking values in $\mathcal{S}_\vlam$ such that
    \[
        \left\|\fwlamu - \left(f_{0,\vhlam,\vhu} + \sum_{(i,\gamma)\in \calJ\backslash \Sundet} \wt{\lambda}_i \,|\iprod{\, \wt{u}_i,\cdot}| \right) - f_{w,\vlam_{S^\star}, \vu_{S^\star}} \right\| \lesim \ulam + k^2 \radius \upsilon.
    \]
    for $\upsilon$ defined in \eqref{eq:newerror}.
\end{lemma}

\begin{proof}
    Take any $(i,\gamma) \in \calJ \backslash\Sundet$. Let $\wt{\lambda}_i\in \mathcal{S}_\vlam$ be the closest point to $\sum_{j\in \Sclose_i(\gamma)}\lambda_j$ in $\mathcal{S}_\vlam$. Then by Lemma~\ref{lem:paramtoL2},
    \begin{align}
        \MoveEqLeft\Bigl\| \sum_{(i,\gamma)\in \calJ\backslash\Sundet} \sum_{j\in\Sclose_i(\gamma)} \lambda_j\,|\iprod{u_j,\cdot}| - \sum_{(i,\gamma)\in\calJ\backslash\Sundet} \wt{\lambda}_i\,|\iprod{\wt{u}_i, \cdot}|\Bigr\| \\
        &\le \sum_{(i,\gamma)\in \calJ\backslash\Sundet} \Bigl\|\wt{\lambda}_i \,|\iprod{\wt{u}_i,\cdot}| - \sum_{j\in \Sclose_i(\gamma)}\lambda_j \,|\iprod{u_j,\cdot}|\Bigr\| \lesim k^2\radius \upsilon\,.
    \end{align}
    The lemma then follows by Lemma~\ref{lem:subtract_linear}, \eqref{eq:l2close}, and triangle inequality.
\end{proof}

\begin{lemma}\label{lem:twooutcomes}
    At the end of one step of our recursive procedure, either we terminate with a function $\wh{f}: \R^d\to \mathbb{R}$ such that
    \[
        \norm{\fwlamu - \wh{f}} \le \epsilon
    \]
    or we obtain a subset $S'\supsetneq S$ and $(\overline{\vlam}, \overline{\vu}) = (\overline{\lambda}_i, \overline{u}_i)_{i \in S'}$ for which 
    \[
        \norm{\fwlamu - f_{0,\overline{\vlam}, \overline{\vu}} - f_{w,\vlam_{(S')^c}, \vu_{(S')^c}}} \lesim k^2\radius\ulam\,.
    \]
\end{lemma}

\begin{proof}
    At the start of the recursion step, suppose we are in Case 1, then by Lemma \ref{lem:twoneuron}, using $\poly(d)(\radius^2 + \norm{\vhlam}^2_1)/\xi'^2$ samples, we can find a function $h$ that is a linear combination of two ReLUs such that
    \[
    \norm{\fwlamu - f_{0,\vhlam,\vhu} - h } \lesim \poly(d, \mathcal{R}) \cdot (\epsilon' + \xi') \lesim \epsilon\,.
    \]
    In this case, in one of the computation paths of our nondeterministic algorithm, it terminates with an $\epsilon$-accurate estimator.
    
    Suppose we are not in Case 1, so that by Lemma~\ref{lem:gap_exists} we are in Case 2 and can find some nonempty set $\calJ\subseteq S^c$ of $(i,\gamma)$ for which $\Sclose_i(\gamma)$ are all disjoint and $\gamma$ is a gapped scale for $i$. As in Lemma~\ref{lem:epscoverlambda}, let $\Sundet\subseteq\calJ$ denote the set of $(i,\gamma)\in\calJ$ such that $i$ is undetectable at scale $\gamma$. Then combining Lemma~\ref{lem:listS} and \ref{lem:epscoverlambda}, using $\poly(d)(\radius^2 + \norm{\vhlam}^2_1)/\xi'^2$ samples for some $\upsilon$ defined in \eqref{eq:newerror} we can (non-deterministically) find $\brc{\wt{\lambda}_i}_{(i,\gamma)\in\calJ\backslash\Sundet}$ such that if we add $(\wt{\lambda}_i, u_i)_{(i,\gamma)\in\calJ\backslash\Sundet}$ and $(0,u_i)_{(i,\gamma)\in\Sundet}$ to $(\vhlam,\vhu)$ to produce $(\overline{\vlam},\overline{\vu})$, and define $S'$ given by $S' \triangleq S \cup \bigcup_{(i,\gamma)\in\calJ} \brc{i}$, then
    \begin{align}
        \norm{\fwlamu - f_{0,\overline{\vlam},\overline{\vu}} - f_{w,\vlam_{(S')^c}, \vu_{(S')^c}}} &\le \ulam + k^2 \radius \upsilon \lesssim k^2\radius\ulam\,. \qedhere\label{eq:l2case2}
    \end{align}
\end{proof}

\noindent In Section~\ref{sec:shortpath}, we show that there is a computation path in our nondeterministic algorithm such that the second outcome in Lemma~\ref{lem:twooutcomes} happens for at most $O(\log k)$ recursive steps before either $S = [k]$ or we arrive at the first outcome. Formally:

\begin{restatable}{lemma}{shortpath}\label{lem:shortpath}
    Given any $(\vlam,\vu)\in(\R\times \S^{d-1})^k$, there exists a sequence of sets $\calJ_1,\ldots,\calJ_q \in [k]\times[\underline{\gamma},1)$ such that the following holds. For every $s\in[q]$, let $I_s$ denote the set of $i\in[k]$ for which there exists $\gamma$ such that $(i,\gamma) \in \calJ_s$. Then
    \begin{enumerate}
        \item $I_1,\ldots,I_q$ are disjoint.
        \item For all $i,j\in[k]\backslash(I_1\cup\cdots \cup I_q)$, $|v_i - v_j| \le \epsilon'$ (if $[k] = I_1\cup\cdots \cup I_q$, this holds vacuously).
        \item For each $s\in[q]$, all of the subsets $\Sclose_i(\gamma)$ for $(i,\gamma)\in\calJ_s$ are disjoint.
        \item For each $s\in[q]$ and each $(i,\gamma)\in\calJ_s$, $\gamma$ is a gapped scale for $i$.
    \end{enumerate}
\end{restatable}

\noindent We conclude with the main guarantee of this section: some computation path of Algorithm~\ref{alg:main} produces an $\epsilon$-accurate estimate for the original function $\flamu$.

\begin{lemma}
    Given $\poly(d,1/\epsilon,\radius)^{k^{O(\log^2 k)}}$ samples and runtime, with high probability over the samples and the randomness of the choice of $g$, there is some computation path in Algorithm~\ref{alg:main} in which the algorithm terminates having found a function $\wh{f}: \R^d \to \mathbb{R}$ such that $\norm{\flamu - \wh{f}} \le \epsilon$.
\end{lemma}

\begin{proof}
    Under the second outcome in Lemma~\ref{lem:twooutcomes}, the $L_2$ error increases from $\xi$ in \eqref{eq:l2close} to $O(k^2\radius\Lambda)$. Lemma~\ref{lem:shortpath} ensures that there is some computation path which terminates after this happens $O(\log k)$ times. Recall that we chose $\Lambda \triangleq \poly(d\radius/\epsilon)\cdot \xi^{1/\Theta\left(k^{\log k}\right)}$, so if this increase is repeated $O(\log k)$ times starting from an initial error of $\xi$, we end up with an estimator that has error at most $\poly(d\radius/\epsilon)\cdot \xi^{1/\Theta(k^{\log^2 k})}$. In particular, by taking the initial $\xi$ to be
    \begin{equation}
        \xi = (\epsilon/\radius d)^{\Theta(k^{\log^2 k})}\,, \label{eq:xidef}    
    \end{equation}
    the final error is bounded by $\epsilon$ as desired. Recall that for a given recursive step, if the initial error in \eqref{eq:l2close} is $\xi$, then the sample complexity is $\poly(d)(\radius^2 + \norm{\wh{\lambda}}^2_1)/\xi'^2 = \poly(d,1/\epsilon,\radius)^{O(k^{\log^2 k})}$ as desired. The number of computation paths is also of this order, as the size of the net $\calS_{\vlam} \times \calS$ used in a single recursive step is at most $O(1 / \upsilon)^{k^2 + k} = \Lambda^{k^2 + k} = \poly(d,1/\epsilon,\radius)^{k^{O(\log^2 k)}}$.
\end{proof}



\subsection{Power sum estimate}
\label{sec:powersum}

In this section we prove a technical claim which is essential to correctness of the PCA-based procedure described in Case 2a from Section~\ref{sec:case2a}.

\begin{lemma}\label{lem:powersum}
    Let $1 \le k' \le k$, and let $q\in\R^k$ be a vector such that $|\sum^{k'}_{i=1} q_i| \ge \tau$ and $\norm{q}_\infty \le R$. If $v\in[-1,1]^k$ satisfies
    \begin{enumerate}
        \item $|v_1| \ge \alpha$
        \item $|v_1 - v_i| \le \beta$ for all $1 \le i \le k'$, 
        \item $|v_i - v_j| \ge \gamma$ for all $1 \le i \le k' < j \le k$.
    \end{enumerate}
    for some parameters $0 < \alpha,\beta,\gamma < 1$, then there exists an even integer $0 \le \ell \le 2k$ for which
    \begin{equation}
        |\iprod{v^{\odot \ell},q}| \gtrsim \frac{\tau}{2k}\Bigl(\frac{\alpha^2\gamma^2}{4k}\Bigr)^k - CRk(k'-1)\beta
    \end{equation} for some absolute constant $C > 0$.
\end{lemma}

To interpret this lemma, it is helpful to first consider the special case where $k' = 1$. In this case, there is a single entry, $v_1$, of non-negligible magnitude which is separated from all other entries by a margin of $\gamma$. The claim is that for any vector $q$ that with a non-negligible first entry, there is some entrywise power of $v$ which has non-negligible correlation with $q$. Importantly, this holds even if the other entries of $v$ are not well-separated. As such one can interpret this result as some kind of \emph{robust local inverse} for the Vandermonde matrix: even if the $k\times k$ Vandermonde matrix generated by the nodes $v^2_1,\ldots,v^2_k$ is ill-conditioned, it is well-conditioned in directions that place sufficient mass on coordinates corresponding to nodes which are separated from the other entries. We also remark that the $\gamma^k$ scaling is qualitatively tight by the following example: consider $v^{\odot 2} = (1,1-\gamma,\ldots, 1-(k-1)\gamma)$ and $q = (\binom{k-1}{0}, -\binom{k-1}{1}, \binom{k-1}{2},\ldots,(-1)^{k-1}\binom{k-1}{k-1})$. Then one can check that $\iprod{v^{\odot \ell},q} = 0$ for all $\ell = 0,\ldots,2k-2$, and for $\ell = k$, this equals $k!\gamma^k$.

For general $k'$, note that the bound in the lemma is non-vacuous provided $\beta$ scales with $\gamma^k$, and this ``gap'' between $\gamma$ and $\gamma^k$ is the central motivation behind our definition of $T(\gamma)$ and the different scales considered in the analysis of Case 2a in Section~\ref{sec:case2a}.

\begin{proof}[Proof of Lemma~\ref{lem:powersum}]
    We first reduce to the case where $k' = 1$. Consider modifying $v,q$ follows. Remove from $v$ entries $2,\ldots,k'$. Also set the first entry of $q$ to be $\sum^{k'}_{i=1} q_i$ and remove from $q$ entries $2,\ldots,k'$. As $\norm{v}_\infty \le 1$ and $\norm{q}_\infty \le R$, this changes every $\iprod{v^{\odot\ell},q}$ by at most
    \begin{equation}
        \sum^{k'}_{i=1} q_i (v^\ell_1 - v^{\ell}_i) = \sum^{k'}_{i=2} q_i (v^\ell_1 - v^\ell_i) \lesim R\ell (k'-1)\beta \lesim Rk(k'-1)\beta\,.
    \end{equation}
    It therefore suffices to prove the lemma for $k' = 1$, so henceforth we specialize to this case.

    Suppose to the contrary that for all $\ell = 0, 2,4,\ldots,2k-2$, $|\iprod{v^{\odot \ell},q}| \le \zeta$ for
    \begin{equation}
        \zeta\triangleq \frac{\tau}{2k}\Bigl(\frac{\alpha^2\gamma^2}{4k}\Bigr)^k\,.
    \end{equation}
    Next, note that
    \begin{equation}
        \iprod{v^{\odot \ell}, q} = \sum^k_{i = 1} q_i v^{\ell}_i = v^\ell_1 \sum^k_{i= 1} q_i (v^2_i / v^2_1)^{\ell/2}, \label{eq:dividebyv1}
    \end{equation}
    so if we define $v'_i \triangleq v^2_i / v^2_1$, then we have that $|\iprod{v'^{\odot \ell/2}, q}| \le \zeta / v_1^\ell \le \zeta \alpha^{-2k}$. Furthermore, for all $j > 1$, we have $|v'_j - 1| \ge \gamma^2 / v^2_1 \ge \gamma^2$. Additionally, $|v'_i| \le 1 / \alpha^2$ for all $i\in[k]$.

    Define $\epsilon_i = v'^2_i - 1$ so that $\epsilon_1 = 0$ and
    \begin{equation}
        |\epsilon_i| \ge \gamma^2
    \end{equation}
    for $i > 1$. Then for all $\ell = 0,2,4,\ldots,2k-2$,
    \begin{equation}
        \iprod{v'^{\odot\ell/2} - \vec{1}, q} = \sum^k_{i=2} q_i ((1 + \epsilon_i)^{\ell/2} - 1) = \sum^k_{i=2} q_i \sum^{\ell/2}_{s = 1} \binom{\ell/2}{s}\epsilon_i^s = \sum^{\ell/2}_{s = 1} \binom{\ell/2}{s} \sum^k_{i=2} q_i \epsilon_i^s.
    \end{equation}
    As $|\iprod{v'^{\odot\ell/2} - \vec{1},q}| \le 2\zeta\alpha^{-2k}$ and
    \begin{align}
        \Bigl|\sum^k_{i=1} q_i \epsilon^{\ell/2}_i\Bigr| \le |\iprod{v'^{\odot\ell/2} - \vec{1},q}| + \sum^{\ell/2-1}_{s=1} \binom{\ell/2}{s} \Bigl|\sum^k_{i=2} q_i \epsilon^s_i\Bigr| \le 2\zeta \alpha^{-2k} + 2^k \sum^{\ell/2 - 1}_{s = 1}\Bigl|\sum^k_{i=2} q_i \epsilon^s_i\Bigr|
    \end{align}
    by induction we find that for all $1 \le s < k$,
    \begin{equation}
        \Bigl|\sum^k_{i=2} q_i \epsilon^s_i\Bigr| \lesim (2k / \alpha^2)^k \cdot \zeta
    \end{equation}
    We also have that 
    \begin{equation}
        \Bigl|\sum^k_{i = 2} q_i\Bigr| \ge |q_1| - \zeta\alpha^{-2k} \ge \tau - \zeta\alpha^{-2k} > \tau/2\,.  \label{eq:sumw2k_big}
    \end{equation}
    We now use Lemma~\ref{lem:vieta} to draw a contradiction. Taking $\ell = k-1$ and $z = (\epsilon_1,\ldots,\epsilon_k)$ in the lemma, we find that
    \begin{equation}
        \sum^k_{i=2} q_i \epsilon^{k-1}_i = \sum^{k-2}_{s= 0} (-1)^{k-s} p_{k - 1 - i}(\epsilon_2,\ldots,\epsilon_k) \cdot \sum^{k}_{i=2} q_i \epsilon^s_i \label{eq:apply_vieta}
    \end{equation}
    where here $p_s$ denotes the elementary symmetric polynomial of degree $s$ on $k - 1$ variables.
    We can rearrange to get
    \begin{equation}
        -p_{k-1}(\epsilon_2,\ldots,\epsilon_k) \sum^{k}_{i=2} q_i = -\sum^k_{i=2} q_i \epsilon^{k-1}_i + \sum^{k-2}_{s= 1} (-1)^{k-s} p_{k - 1 - i}(\epsilon_2,\ldots,\epsilon_k) \cdot \sum^{k-2}_{i=1} q_i \epsilon^s_i \label{eq:apply_vieta2}
    \end{equation}
    As $|\epsilon_i| \in [\gamma^2,1]$, $p_{k-1-s}(\epsilon_2,\ldots,\epsilon_k) \le 2^k$ for all $s$, and $|p_{k-1}(\epsilon_2,\ldots,\epsilon_k)| \ge \gamma^{2k}$. So by triangle inequality,
   \begin{equation}
       \Bigl|\sum^k_{i=2} q_i\Bigr| \le k\, \Bigl(\frac{4k}{\alpha^2\gamma^2}\Bigr)^k \cdot \zeta \le \tau/2,
   \end{equation}
   which contradicts \eqref{eq:sumw2k_big}.
\end{proof}

\noindent The above proof used the following basic fact:
\begin{lemma}\label{lem:vieta}
    Given $z = (z_1,\ldots,z_K)$, let $v_\ell$ denote the vector $(z^\ell_1,\ldots,z^\ell_K)$. Let $p_\ell$ denote the elementary symmetric polynomial of degree $\ell$ on $K$ variables. Then
    \begin{equation}
        v_{K} = \sum^{K-1}_{s=0} (-1)^{K - s + 1} p_{K - s}(z)\cdot v_s
    \end{equation}
\end{lemma}

\subsection{Validation}
Now we have a set of candidate estimators, one of which is guaranteed to be close to the ground truth network in square loss. We will use a validation set of fresh samples to estimate the loss of each of these estimators and pick the best predictor. In order to guarantee that this predictor will have low test loss, we will prove a concentration property of ReLU networks (as in \cite{chen2022learning}). Notice that for any estimator function $f_{\vlam, \vec{u}}$, its difference with the ground truth network $f^*$, $f_{\vlam, \vec{u}}-f^*$ has width $2k$ and is a $2k\radius$-Lipschitz function. 

\begin{lemma}
\label{lemma:validation}
For an arbitrarily given $\delta > 0$, and $t < 4\radius ^2 k$. Let $F:\mathbb{R}^d\rightarrow \mathbb{R}$ be a $2\radius $-Lipschitz one-hidden-layer ReLU network with width $2k$. Then, for $N$ i.i.d samples $x_1, x_2, \ldots, x_N \sim \mathcal N(0,I_d)$, where $N = \Theta\left((\mu+4\radius ^2k)^2\log(1/\delta)/t^2\right)$. Here $\mu := \mathbb{E}_{x\sim \mathcal N(0,I_d)}[F(x)]$. Denote the empirical estimate of the squared loss $\widehat{\sigma}^2 :=\frac1N\sum_{i=1}^N F(x_i)^2$, then with probability $1- \delta$,
\[\left|\mathbb{E}_{x\sim\mathcal N(0,I_d)}\left[F(x)^2\right] - \widehat{\sigma}^2\right|\leqslant t.\]
\end{lemma}

\noindent We defer the proof of this to Appendix~\ref{app:validate}. 

Since at each step we branch out by a factor of $\poly(d,\radius,1/\epsilon)^{k^{O(\log^2 k)}}$, and we run for $k$ steps, we have $\poly(d,\radius,1/\epsilon)^{k^{O(\log^2 k)}}$ predictors to test from. Using the above lemma with $t = \epsilon$, followed by a union bound, with a validation set of size $\approx k^{O(\log^2 k)}\poly(\radius, 1/\epsilon) \log (d/\delta)$ with probability $1-\delta$, we will find a good predictor up to an additive error of $\epsilon$. This completes the proof of Theorem~\ref{thm:main}.

\section{Terminating in \texorpdfstring{$O(\log k)$}{O(log k)} steps}
\label{sec:shortpath}

In this section we prove Lemma~\ref{lem:shortpath}, which ensures that there is a successful computation path in our algorithm of length $O(\log k)$. We restate the lemma here for convenience,

\shortpath*

Before proving this lemma in Section~\ref{sec:shortpath_proof}, we first identify a particular game in Section~\ref{sec:game} and upper bound the smallest number of steps needed to win this game. We then show in Section~\ref{sec:shortpath_proof} how this bound implies an upper bound on the shortest successful computation path in our algorithm.

\subsection{Clumping game}
\label{sec:game}

This section can be read independently of the rest of this work, and the notational choices here are specific to this section.

Consider the following game. Let $k\in\mathbb{N}$ and let $\tau,\phi\ge 0$. We start with a vector $w\in \R^k_{\ge 0}$ for which $w_1 = w_k = 0$. At every step, we say that a \emph{$\phi$-legal move in the (noiseless) clumping game} is a move consisting of the following steps:
\begin{enumerate}
    \item Select a sequence of indices $1= i_1 < j_1 \le i_2 < j_2 \le \cdots \le i_m < j_m = k$ such that for every $a\in[m]$, we have that $w_{i_a}, w_{j_a} \le \tau$ and furthermore at least one of the following two conditions holds:
    \begin{itemize}
        \item $w_\ell \le \tau$ for all $i_a < \ell < j_a$.
        \item $w_\ell > \max(w_{i_a}, w_{j_a}) + \phi$ for all $i_a < \ell < j_a$.
    \end{itemize}
    In this case, we say that the interval $[i_a,j_a]$ is \emph{good} for $w$. Note that if $j_a = i_a + 1$ and $w_{i_a}, w_{j_a} \le \tau$, then it is vacuously true that $[i_a,j_a]$ is good for $w$.
    \item Suppose that the union of the intervals $[i_1,j_1],\ldots,[i_m,j_m]$, \emph{regarded as subsets of $\R$}, is equal to the union of intervals $[i^*_1,j^*_1],\ldots,[i^*_n,j^*_n]$ for $n \le m$ such that $j^*_a < i^*_{a+1}$ for all $a$. Then for every $a$, replace the entries $i^*_a, i^*_a+1, \ldots, j^*_a$ of $w$ with the single entry $\min_{i^*_a \le \ell \le j^*_a} w_{\ell}$.
    \item Update $k$ to be the length of the resulting vector
\end{enumerate}

Steps 1 to 3 altogether count as a single move. The game ends when no more $\phi$-legal moves are possible, e.g. when $k = 1$. We remark that the distinction between $\brc{[i_a,j_a]}$ versus $\brc{[i^*_a,j^*_a]}$ in Step 2 will only be relevant in one place in the proof (see the footnote in the proof of Lemma 5.5). Otherwise, the moves we make will be such that $j_a < i_{a+1}$ for all $a$, so that there is no difference between $\brc{[i_a,j_a]}$ and $\brc{[i^*_a,j^*_a]}$.

It is not hard to see that for $\phi = \Omega(1)$ and $\tau \gtrsim \Omega(\log k)$, there always exists a $\phi$-legal move that decreases $k$ as long as $k > 1$, so this game will always terminate in at most $k - 1$ moves, and furthermore, by design, the final vector will consist of a single zero entry. The proof of this is essentially identical to the proof of Lemma~\ref{lem:gap_exists}.

We will show the following stronger guarantee:

\begin{lemma}\label{lem:puzzle_main}
    For sufficiently large absolute constants $c, C > 0$, the following holds. For $\tau = c\,\log k$, starting at an arbitrary $w\in\R^k_{\ge 0}$ for which $w_1 = w_k = 0$, there is a sequence of at most $C\, \log k$ moves, each of them $1$-legal, after which the game will end, with the final vector consisting of a single zero entry.
\end{lemma}

\noindent We first introduce some terminology:

\begin{definition}
    Given $r\in\mathbb{N}$ and $i_1,j_1,\ldots,i_m,j_m\in[r]$, 
    we say that intervals $I_1 = [i_1,j_1],\ldots,I_m = [i_m,j_m]$ form a \emph{separated partition ${\mf P}$ of $[r]$ of size $m$} if
    \begin{equation}
        1= i_1\le j_1 < i_2 \le j_2 < \cdots < i_m \le j_m = k
    \end{equation}
    and furthermore $j_{a+1} > i_a + 1$ for all $1 \le a < m$. For separated partitions, we will refer to the intervals $\brc{[j_a + 1,i_{a+1} - 1]}_{1\le a < m}$ as the \emph{gaps} of ${\mf P}$.
\end{definition}

\begin{example}
    The intervals $[1,2], [4,7], [10,12]$ form a separated partition of $[12]$, but the intervals $[1,2], [3,7], [10,12]$ do not. For the former, the gaps of the partition are given by the intervals $[3,3]$ and $[8,9]$.
\end{example}

\noindent The following trivial observation will be essential to the proof of Lemma~\ref{lem:puzzle_main} below:

\begin{lemma}\label{lem:half}
    Any separated partition of a set $[r]$ has size at most $\ceil{r/2}$.
\end{lemma}

\noindent Finally, we note that the following can be achieved with two moves.

\begin{lemma}\label{lem:twomove}
    Suppose there is a sequence of indices $1= i_1\le j_1 \le \cdots \le i_m \le j_m = k$ such that for every $a\in[m]$, we have that $w_{i_a}, w_{j_a} \le \tau$ and furthermore there is some $\tau' \le \tau$ such that for every $i_a \le \ell \le j_a$, we either have $w_\ell \le \tau'$ or $w_\ell > \tau' + 1$. We say that the intervals $[i_a, j_a]$ are \emph{moderate} for $w$.
    
    Let $w^*$ be the vector obtained by replacing all of the entries in $w$ indexed by $[i_a,j_a]$ with $\min_{i_a \le \ell \le j_a} w_\ell$, for every $a$. Then $w^*$ can be obtained from $w$ in two moves.
\end{lemma}

\begin{proof}
    For every $a$, let $[r^{(a)}_1,s^{(a)}_1], \ldots, [r^{(a)}_{m_a}, s^{(a)}_{m_a}]$ denote the separated partition of $[i_a,j_a]$ into intervals such that entries of $w$ corresponding to indices within an interval strictly exceed $\tau' + 1$, and such that entries of $w$ corresponding to indices within a gap of this partition are at most $\tau'$. Then define $i^{(a)}_c = r^{(a)}_c - 1$ and $j^{(a)}_c = s^{(a)}_c + 1$ so that $[i^{(a)}_c, j^{(a)}_c]$ is good for $w$. 

    We can make one move using all of the intervals $[i^{(a)}_c, j^{(a)}_c]$.\footnote{This is the only part of the proof in this section where the intervals defining the move are not disjoint as subsets of $\R$, so that there is a distinction between $[i,j]$ and $[i^*,j^*]$ in Step 2.} Note that the resulting vector in the next step of the game, call it $w'$, can equivalently be defined by taking $w$ and, within every interval $[i_a,j_a]$ of coordinates, removing all entries of $w$ which exceed $\tau' + 1$ as well as some other entries that are not the minimum entry of $w$ within that interval.

    Every interval $[i_a,j_a]$ of coordinates from $w$ corresponds in a natural way to an interval $[i'_a,j'_a]$ of coordinates from $w'$. Note that within any such interval, the entries of $w'$ are at most $\tau'$ by design, so $[i'_a,j'_a]$ is good for $w'$. We can then make one move using all of the intervals $[i'_a,j'_a]$ to obtain the vector $w^*$ defined in the lemma.
\end{proof}

\begin{lemma}\label{lem:zeroentry}
    If at any point in the game there is exactly one zero entry in the vector, then the game is over.
\end{lemma}

\begin{proof}
    By design, the leftmost and rightmost entries of any vector produced over the course of the game must be zero. So if there is a single zero entry, this means the vector is one-dimensional, and thus the game has ended.
\end{proof}

\begin{proof}[Proof of Lemma~\ref{lem:puzzle_main}]
    Starting with $s = 0$, we inductively define the following objects. Let $k^{(0)} \triangleq k$, $w^{(0)} \triangleq w$, and $u^{(0)} \triangleq w$. We will maintain the invariant that $u^{(s)}$ is a subsequence of $w^{(s)}$ such that all entries outside of $u^{(s)}$ are $> \tau - s + 1$.
    
    Let $\mf{P}_s$ be a separated partition of $[k^{(s)}]$, and denote its size by $k^{(s+1)}$. Suppose that $\mf{P}_s$ consists of intervals $I^{(s)}_1,\ldots,I^{(s)}_{k^{(s+1)}}$ such that:
    \begin{itemize}
        \item Every entry of $u^{(s)}$ indexed by an entry from one of these intervals is $\le \tau - s$
        \item Every entry of $u^{(s)}$ indexed by an entry from a gap of $\mf{P}_s$ is $> \tau - s$.
    \end{itemize}
    
    Let $1 \le a^{(s+1)}_1 < \cdots < a^{(s+1)}_{k^{(s+1)}}\le k^{(s)}$ denote the indices in $I^{(s)}_1,\ldots, I^{(s)}_{k^{(s+1)}}$ over which $u^{(s)}$ is minimized within those intervals (breaking ties arbitrarily). Then define $u^{(s+1)} \in \R^{k^{(s+1)}}_{\ge 0}$ to be the entries of $u^{(s)}$ indexed by $a^{(s+1)}_1, \ldots, a^{(s+1)}_{k^{(s+1)}}$.

    Finally, we describe how to define $w^{(s+1)}$. First note that each interval $I^{(s)}_\ell$ corresponds to some subset $\brc{b_{\ell,1},\ldots, b_{\ell,m_\ell}}$ of the entries of $w^{(s)}$. Consider the intervals $[b_{\ell,1}, b_{\ell, m_\ell}]$ for every $\ell$, regarded as subsets of the coordinates of $w^{(s)}$.
    
    We claim that these intervals are all moderate for $w^{(s)}$ in the sense of Lemma~\ref{lem:twomove}. By the inductive hypothesis, all entries of $w^{(s)}$ outside of $u^{(s)}$ are $> \tau - s + 1$. So within any interval $[b_{\ell,1}, b_{\ell, m_\ell}]$ of coordinates of $w^{(s)}$, the indices which are not $b_{\ell,c}$ for some $c$ are $> \tau - s + 1$, whereas the indices which are $b_{\ell, c}$ for some $c$ are, by assumption on $I^{(s)}_\ell$, at most $\tau - s$. Therefore, the intervals $[b_{\ell,1}, b_{\ell, m_\ell}]$ are moderate for $w^{(s)}$ as claimed.
    
    We can thus make two moves in the game to replace the entries in each interval $[b_{\ell,1}, b_{\ell, m_\ell}]$ of coordinates in $w^{(s)}$ by the single entry $u^{(s+1)}_\ell$. Define $w^{(s+1)}$ to be this new vector. Note that the entries of $w^{(s+1)}$ outside of $u^{(s+1)}$ were either entries of $u^{(s)}$ from among the gaps of $\mf{P}_s$, or entries of $w^{(s)}$ outside of $u^{(s)}$. In the former case, by assumption on $\mf{P}_s$, such entries are $> \tau - s$, and in the latter case, by the inductive hypothesis, such entries are $> \tau - s + 1$. We have thus maintained the invariant that $u^{(s+1)}$ is a subsequence of $w^{(s+1)}$ such that all entries of $w^{(s+1)}$ outside of $u^{(s+1)}$ are $> \tau - (s + 1) +1$.

    Note that by Lemma~\ref{lem:half}, $k^{(t)} \le k^{(t-1)}/2$. So after making at most $O(\log k)$ moves as defined above, we end up with a vector $w^{(t)}$ which has exactly one entry which is at most $\tau - t$. If $\tau \ge \Omega(\log k)$, then because there will always be a zero entry within any vector obtained over the course of the game, this means that there is exactly one zero entry in $w^{(t)}$. By Lemma~\ref{lem:zeroentry}, this means the game has ended.
\end{proof}

\subsection{Noisy clumping game}

As we will see in Section~\ref{sec:shortpath_proof}, the steps in our learning algorithm will only \emph{approximately} correspond to moves in the clumping game. More precisely, the former exactly correspond to the following ``noisy'' version of the clumping game.

\begin{definition}
    Given $w\in\R^k_{\ge 0}$, the vector $w'\in\R^k_{\ge 0}$ is a $\Delta$-perturbation of $w$ if for every $\ell\in[k]$, the following holds:
    \begin{itemize}
        \item $w'_\ell \le w_\ell$
        \item If additionally $w_\ell > 1$, then $w'_\ell \ge w_\ell - \Delta$.
    \end{itemize}
\end{definition}

Given $k\in\mathbb{N}$, $\tau,\phi\ge 0$, and starting vector $w\in\R^k_{\ge 0}$ for which $w_1 = w_k = 0$, we say that a \emph{$\phi$-legal move in the \textbf{noisy} clumping game} is a move consisting of the following four steps. The first three steps are identical to those of the noiseless clumping game. Then in the fourth step,
\begin{enumerate}
    \setcounter{enumi}{3}
    \item Replace $w$ with an arbitrary (possibly adversarially chosen) $1/100k$-perturbation of $w$.
\end{enumerate}

Lemma~\ref{lem:monotone} ensures that a $1$-legal strategy for the noiseless clumping game can be converted into a $0.99$-legal strategy for the noisy clumping game:

\begin{lemma} \label{lem:noisyvsnoiseless}
    Let $w\in\R^k_{\ge 0}$ be any vector for which $w_1 = w_k = 0$. If there is a sequence of $N$ $1$-legal moves in the noiseless clumping game starting from $w$ after which the game ends with the final vector consisting of a single zero entry, then there is a sequence of $N$ $(1 - \Delta)$-legal moves starting from $w$ in the noisy clumping game after which the game ends with the final vector consisting of a single zero entry.
\end{lemma}

\noindent This follows immediately by inducting on $k$ and repeatedly applying the following:

\begin{lemma} \label{lem:monotone}
    Let $w\in\R^k_{\ge 0}$ be any vector for which $w_1 = w_k = 0$, and let $w'$ be any $\Delta$-perturbation of $w$. Any $1$-legal move in the noiseless clumping game starting from $w$ is also a $(1 - \Delta)$-legal move in the noisy clumping game starting from $w'$.
    
    Let $u, u'$ be the vectors resulting from the former and latter respectively. Then $u'$ is a $(\Delta + 1/100k)$-perturbation of $u$.
\end{lemma}

\begin{proof}
    Suppose the $1$-legal move in the noiseless clumping game is specified by intervals $\brc{[i_a,j_a]}_{a\in[m]}$. 

    Certainly for any $a\in[m]$, if $w_{i_a}, w_{j_a} \le \tau$, then $w'_{i_a}, w'_{j_a}\le \tau$. Likewise, if $w_\ell \le \tau$ for all $i_a < \ell < j_a$, then $w'_\ell \le \tau$ for all $i_a < \ell < j_a$. Otherwise, if $w_\ell > \max(w_{i_a}, w_{j_a}) + 1$ for all $i_a < \ell < j_a$, then by the definition of $\Delta$-perturbation, $w'_\ell \in [w_\ell - \Delta,w_\ell]$, so $w'_\ell > \max(w_{i_a}, w_{j_a}) + 1 - \Delta \ge \max(w'_{i_a}, w'_{j_a}) + 1$ for all $i_a < \ell < j_a$. We conclude that $\brc{[i_a,j_a]}_{a\in[m]}$ is a $(1 - \Delta)$-legal move. The last part of the lemma follows from the fact that a $1/100k$-perturbation of a $\Delta$-perturbation is a $(\Delta + 1/100k)$-perturbation.
\end{proof}

\subsection{Relating the noisy clumping game to Lemma~\ref{lem:shortpath}}
\label{sec:shortpath_proof}

We show that any sequence of $\calJ_1,\ldots,\calJ_q$ satisfying the four conditions of Lemma~\ref{lem:shortpath} corresponds to a sequence of $0.99$-legal moves in the noisy clumping game, and vice versa, after which we can conclude the proof of Lemma~\ref{lem:shortpath} by invoking Lemma~\ref{lem:puzzle_main} and Lemma~\ref{lem:noisyvsnoiseless}.

First, we need the following definition:

\begin{definition}
    Let $c$ be the constant in the exponent of $\gamma/d$ in the definition of $T(\gamma)$ in \eqref{eq:Tdef}. Given $\gamma \in [0,1)$, define the \emph{level of $\gamma$}, denoted $L(\gamma)$, by \begin{equation}
        L(\gamma) \triangleq \frac{0.9}{\ln(ck)} \ln \Bigl(1 + \frac{(ck - 1)\ln(\epsilon'/\gamma)}{ck\ln d + \ln(\radius^2/\ulam^{10}) + (ck - 1)\ln(1/\epsilon')} \Bigr)
    \end{equation}
    where $\Lambda$ is defined in \eqref{eq:paramdef} and $\xi$ therein is given by \eqref{eq:xidef}. This is clearly monotonically increasing as $\gamma$ decreases.
    
    The function $L$ is chosen so that 
    \begin{equation}
        L(T(\gamma)) = 0.9 + L(\gamma)\,.
    \end{equation} In particular, the level of $T^{(n)}(\epsilon')$ is precisely $0.9n$, where $T^{(n)}$ denotes $n$-fold composition of $T$.
\end{definition}

\begin{observation}\label{obs:margin}
    For any $0 \le \gamma_1 \le \cdots \le \gamma_s$, if $L(\gamma_s) > 1$, then
    \begin{equation}
        L(\gamma_s) - O(1/(k\ln d)) \le L(\gamma_1 + \cdots + \gamma_s) \le L(\gamma_s)\,.
    \end{equation} The latter bound also holds if $L(\gamma_s) \le 1$.
\end{observation}

\begin{proof}
    That $L(\gamma_1 + \cdots + \gamma_s) \le L(\gamma_s)$ follows immediately from the fact that $L$ is a monotonically decreasing function. For the other bound, for convenience, denote $\beta \triangleq ck\ln d + \ln(\radius^2/\Lambda^{10} + (ck - 1)\ln(1/\epsilon')$ and note that $L(\gamma_s) \ge 1$ implies that $\ln(1 + (ck-1)\ln(\epsilon'/\gamma)/\beta) \ge ck$ so that
    \begin{equation}
        \ln\frac{1 + \ln(1/k\gamma)\cdot (ck - 1)/\beta}{1 + \ln(1/\gamma)\cdot (ck-1)/\beta} 
        \ge \ln\Bigl(1 - \frac{\ln(k) \cdot (ck-1)/\beta}{ck} \Bigr) \ge 1 - O(\ln k/(k\ln d))\,,
    \end{equation}
    so
    \begin{equation}
        L(\gamma_1 + \cdots + \gamma_s) \ge L(k\gamma_s) \ge L(\gamma_s) - O(1/(k \ln d))\,.
    \end{equation}
\end{proof}

\begin{proof}[Proof of Lemma~\ref{lem:shortpath}]
    Take the ``$k$'' in Section~\ref{sec:game} to be $k + 1$, and let the initial vector ``$w$'' be defined as follows. As specified in Section~\ref{sec:game}, we take its first and last entries to be $0$. For $1 < i < k$, let the $i$-th entry of ``$w$'' from Section~\ref{sec:game} be given by 
    \begin{equation}
        w_i \triangleq L(v_i - v_{i-1})\,.
    \end{equation}
    Define
    \begin{equation}
        \tau \triangleq L(\underline{\gamma}) = \Theta(\log k)\,.
    \end{equation}
    
    For $\phi = 0.99$, consider any first move $\brc{[i_a,j_a]}$ in the noisy clumping game, starting from the vector $w$. 
    
    Suppose first that this move does not consist of $\brc{[1,k+1]}$. We show that this move corresponds to a choice of $\calJ$ satisfying the conditions that
    \begin{itemize}
        \item[(i)] All the subsets $\Sclose_i(\gamma)$ for $(i,\gamma)\in\calJ$ are disjoint.
        \item[(ii)] For each $(i,\gamma)\in\calJ$, $\gamma$ is a gapped scale for $i$.
    \end{itemize}
    For each $a\in[m]$,
    \begin{itemize}
        \item[(A)] \underline{If $w_\ell \le \tau$ for all $i_a \le \ell \le j_a$}: by definition of $w$, this means that $L(v_\ell - v_{\ell-1}) \le \tau$ for all $\max(i_a,2) \le \ell \le \min(j_a,k)$.
        So $v_{\max(i_a-1,1)},\ldots, v_{\min(j_a,k)}$ are all separated by a distance of at least $\underline{\gamma}$. This means that $\Sclose_\ell(\underline{\gamma}) = \brc{\ell}$ while $\Sfar_\ell(\underline{\gamma}) = [k]\backslash\brc{\ell}$, so $\underline{\gamma}$ is a gapped scale for every $i_a \le \ell < j_a$. Add $(\ell,\underline{\gamma})$ for all $i_a \le \ell < j_a$ to $\calJ$.
        \item[(B)] \underline{If $w_{\ell} > \max(w_{i_a},w_{j_a}) + 0.99$ for all $i_a < \ell < j_a$}: this means that 
        \begin{equation}
            v_\ell - v_{\ell - 1} \ll \min(T(v_{i_a} - v_{i_a - 1}), T(v_{j_a} - v_{j_a - 1}))
        \end{equation}
        for all $\max(i_a,2) \le \ell \le \min(j_a,k)$.
        In fact, because of the margin between $\phi = 0.99$ and the constant $0.9$ in the definition of $L$ and by Observation~\ref{obs:margin}, this ensures that 
        \begin{equation}
            v_\ell - v_{\ell'} \ll \min(T(v_{i_a} - v_{i_a - 1}), T(v_{j_a} - v_{j_a - 1}))   
        \end{equation}
        for all $i_a\le \ell < j_a$. 
        
        Let $\gamma$ be such that $L(\gamma) = \min(w_{i_a}, w_{j_a})$. If the minimum is achieved by the former (resp. the latter), then $\gamma$ is a gapped scale for $i_a$ (resp. $j_a$) and $\Sclose_{i_a}(\gamma) = \brc{i_a,\ldots,j_a - 1}$ (resp. $\Sclose_{j_a}(\gamma) = \brc{i_a,\ldots,j_a - 1}$). Add $(i_a,\gamma)$ (resp. $(j_a,\gamma)$) to $\calJ$.
    \end{itemize}
    It is clear from the above construction of $\calJ$ that the two conditions (i) and (ii) hold. 
    
    Finally, we need to relate the removal of neurons indexed by $\cup_{(i,\gamma)\in\calJ} \Sclose_i(\gamma)$ from $[k]$ to Steps 2 and 4 of the clumping game. Indeed, for each $(i,\gamma)\in\calJ$, if the pair came from
    case (A) above and corresponds to an index $\ell$ for which $i_a \le \ell < j_a$, then $\Sclose_i(\gamma) = \brc{\ell}$. Recall that for all $[i_a,j_a]$ in case (A), there is a pair in $\calJ$ corresponding to each such $\ell$. On the other hand, if $(i,\gamma)$ instead came from case (B) above and $[i_a,j_a]$ is the corresponding interval from the move of the noisy clumping game, then $\Sclose_i(\gamma)\subset[k]$ is the set $\brc{i_a,\ldots,j_a - 1}$. Putting things together, we conclude that $\cup_{(i,\gamma)\in\calJ}\Sclose_i(\gamma) = \cup_a \brc{i_a,\ldots,j_a - 1}$. So if $[i^*_1,j^*_1],\ldots,[i^*_n,j^*_n]$ are the intervals defined in Step 2 of the game, then $\cup_{(i,\gamma)\in\calJ}\Sclose_i(\gamma) = \cup_a \brc{i^*_a,\ldots,j^*_a - 1}$. 
    
    Note that for every $1 \le a < n$, $L(v_{j^*_a + 1} - v_{i^*_a - 1}) \le L(\min_{i^*_a \le \ell \le j^*_a} w_\ell)$ and additionally, if $L(\min_{i^*_a \le \ell \le j^*_a} w_\ell) \ge 1$, then $L(v_{j^*_a + 1} - v_{i^*_a - 1}) \ge L(\min_{i^*_a \le \ell \le j^*_a} w_\ell) - O(1/(k\log d))$. For $d$ greater than a sufficiently large constant, the quantity $O(1/(k\log d))$ is at most $1/100k$, so if $\ell_1,\ldots,\ell_s$ are the neurons remaining after the first step of the algorithm, the vector with entries consisting of $L(v_{\ell_a} - v_{\ell_{a - 1}})$ is a valid $1/100k$ perturbation of the vector obtained from Step 2 of the clumping game.
    
    We have thus shown that any initial $0.99$-legal move $\brc{[i_a,j_a]}$ in the noisy clumping game that does not consist solely of the interval $[1,k+1]$ corresponds to a valid set $\calJ$ of neurons that can be learned and removed from consideration in the first iteration of our recursive learner, and moreover the sequence of pairwise separations between successive $v_i$'s associated to the remaining neurons corresponds to the new vector $w$ in Step 4 from making the move $\brc{[i_a,j_a]}$ in the noisy clumping game. See Figure~\ref{fig:clump} for an illustration of this correspondence.

    \begin{figure}[h]
        \centering
        \includegraphics[width=0.9\textwidth]{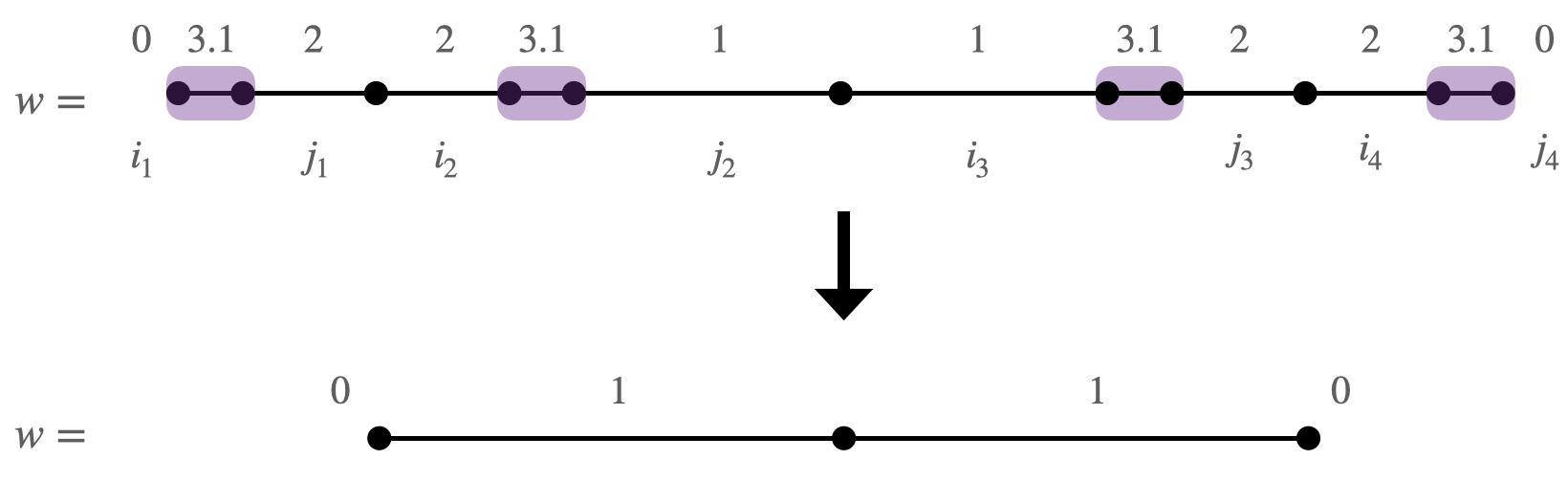}
        \caption{Example of a $1$-legal move in the (noiseless) clumping game, with $\tau = 3$. Initially, the vector is given by $w = (0,3.1,2,2,3.1,1,1,3.1,2,2,3.1,0)$, and the move is given by $(i_1,j_1) = (1,3), (i_2,j_2) = (4,6), (i_3,j_3) = (7,9), (i_4,j_4) = (10,12)$. The vector resulting from this move is $(0,1,0)$. Below the entries of $w$ is a sequence of points $v_i$ such that the distance between the $(i-1)$-st and $i$-th point is $\gamma$ satisfying $L(\gamma) = w_i$. The highlighted regions correspond to clumps of neurons for which there is a gapped scale and which can thus be learned using the analysis in Case 2a in Section~\ref{sec:case2a}. After these neurons are learned and subtracted out of the network, the remaining three neurons have pairwise separations $\gamma_1,\gamma_2$ for which $(L(\gamma_1), L(\gamma_2)) \approx (1,1)$.} 
        \label{fig:clump}
    \end{figure}
    
    The above reasoning can then be applied in an identical fashion to subsequent moves in the clumping game / subsequent iterations of the recursive learner, provided the move does not consist solely of the interval $[1,k+1]$.

    Lemma~\ref{lem:puzzle_main} and Lemma~\ref{lem:noisyvsnoiseless} ensure the existence of a sequence of $0.99$-legal moves in the noisy clumping game at the end of which the game ends with the vector $w$ consisting of a single zero entry. The last move in this sequence must be the move $[1,k+1]$. If $\ell_1,\ldots,\ell_s\in[k]$ are the only remaining neurons prior to this final move, the move $[1,k+1]$ is only $0.99$-legal provided that $L(v_{\ell_a} - v_{\ell_{a-1}}) > 0.99$, but in this case $|v_{\ell_a} - v_{\ell_b}| \le k\cdot T(1) \ll \epsilon'$ for all $a,b\in[s]$, as desired.
\end{proof}


\section{Conclusion and Future Directions}
In this paper, we provided the first PAC learning algorithm that learns narrow (constant $k$) one-hidden-layer neural networks in polynomial time. Our algorithm used random contractions of higher order moment tensors and subsequently performed an iterative procedure to repeatedly learn some clumps of neurons. This allowed us to avoid depending on the condition number of the weight matrix, which was a core assumption in several prior works. 

One obvious drawback of our technique is the unsatisfactory dependence on $k$ in our runtime $d^{k^{\mathcal O(\log^2 k)}}$. Note that known CSQ hardness results rule out the possibility of improving to $d^{o(k)}$ for our algorithm, but we leave as an interesting open question closing the gap between our upper bound and these lower bounds. Furthermore, harnessing the power of non-CSQ style approaches as in \cite{chenmeka,chen2022learning} could potentially allow us to get an algorithm that runs in time $\poly(d)\cdot (1/\epsilon)^{h(k)}$ for some function $h$. 

Another worthwhile direction is to investigate other applications of the power sum estimate technique proposed in our work to tensor problems in the absence of separation conditions.

\bibliographystyle{alpha}
\bibliography{reference}

\newcommand{\etalchar}[1]{$^{#1}$}
\begin{thebibliography}{GMOV18}

\bibitem[AZLL19]{azll}
Zeyuan Allen-Zhu, Yuanzhi Li, and Yingyu Liang.
\newblock Learning and generalization in overparameterized neural networks,
  going beyond two layers.
\newblock In {\em Advances in neural information processing systems}, pages
  6158--6169, 2019.

\bibitem[BJW19]{bakshi2019learning}
Ainesh Bakshi, Rajesh Jayaram, and David~P Woodruff.
\newblock Learning two layer rectified neural networks in polynomial time.
\newblock In {\em Conference on Learning Theory}, pages 195--268. PMLR, 2019.

\bibitem[CKM21]{chen2021efficiently}
Sitan Chen, Adam~R Klivans, and Raghu Meka.
\newblock Efficiently learning any one hidden layer relu network from queries.
\newblock {\em arXiv preprint arXiv:2111.04727}, 2021.

\bibitem[CKM22]{chen2022learning}
Sitan Chen, Adam~R Klivans, and Raghu Meka.
\newblock Learning deep relu networks is fixed-parameter tractable.
\newblock In {\em 2021 IEEE 62nd Annual Symposium on Foundations of Computer
  Science (FOCS)}, pages 696--707. IEEE, 2022.

\bibitem[CLLZ22]{chen2022learningpoly}
Sitan Chen, Jerry Li, Yuanzhi Li, and Anru~R Zhang.
\newblock Learning polynomial transformations.
\newblock {\em arXiv preprint arXiv:2204.04209}, 2022.

\bibitem[CM20]{chenmeka}
Sitan Chen and Raghu Meka.
\newblock Learning polynomials of few relevant dimensions.
\newblock {\em arXiv preprint arXiv:2004.13748}, 2020.

\bibitem[Dan17]{Daniely17}
Amit Daniely.
\newblock Sgd learns the conjugate kernel class of the network.
\newblock {\em CoRR}, abs/1702.08503, 2017.

\bibitem[DGK{\etalchar{+}}20]{diakonikolas2020approximation}
Ilias Diakonikolas, Surbhi Goel, Sushrut Karmalkar, Adam~R Klivans, and Mahdi
  Soltanolkotabi.
\newblock Approximation schemes for relu regression.
\newblock In {\em Conference on Learning Theory}, 2020.

\bibitem[DK20]{smallcovers}
Ilias Diakonikolas and Daniel~M Kane.
\newblock Small covers for near-zero sets of polynomials and learning latent
  variable models.
\newblock In {\em 2020 IEEE 61st Annual Symposium on Foundations of Computer
  Science (FOCS)}, pages 184--195. IEEE, 2020.

\bibitem[DKKZ20]{diakonikolas2020algorithms}
Ilias Diakonikolas, Daniel~M Kane, Vasilis Kontonis, and Nikos Zarifis.
\newblock Algorithms and sq lower bounds for pac learning one-hidden-layer relu
  networks.
\newblock In {\em Conference on Learning Theory}, pages 1514--1539, 2020.

\bibitem[DV20]{daniely2020hardness}
Amit Daniely and Gal Vardi.
\newblock Hardness of learning neural networks with natural weights.
\newblock {\em arXiv preprint arXiv:2006.03177}, 2020.

\bibitem[GGJ{\etalchar{+}}20]{goel2020superpolynomial}
Surbhi Goel, Aravind Gollakota, Zhihan Jin, Sushrut Karmalkar, and Adam
  Klivans.
\newblock Superpolynomial lower bounds for learning one-layer neural networks
  using gradient descent.
\newblock In {\em International Conference on Machine Learning}, pages
  3587--3596. PMLR, 2020.

\bibitem[GK19]{goel2019learning}
Surbhi Goel and Adam~R Klivans.
\newblock Learning neural networks with two nonlinear layers in polynomial
  time.
\newblock In {\em Conference on Learning Theory}, pages 1470--1499, 2019.

\bibitem[GKKT17]{goel2017reliably}
Surbhi Goel, Varun Kanade, Adam Klivans, and Justin Thaler.
\newblock Reliably learning the relu in polynomial time.
\newblock In {\em Conference on Learning Theory}, pages 1004--1042. PMLR, 2017.

\bibitem[GKLW18]{ge2018learning2}
Rong Ge, Rohith Kuditipudi, Zhize Li, and Xiang Wang.
\newblock Learning two-layer neural networks with symmetric inputs.
\newblock In {\em International Conference on Learning Representations}, 2018.

\bibitem[GKM18]{convotron}
Surbhi Goel, Adam~R. Klivans, and Raghu Meka.
\newblock Learning one convolutional layer with overlapping patches.
\newblock In {\em ICML}, volume~80, pages 1778--1786. PMLR, 2018.

\bibitem[GLM18]{ge2018learning}
Rong Ge, Jason~D Lee, and Tengyu Ma.
\newblock Learning one-hidden-layer neural networks with landscape design.
\newblock In {\em 6th International Conference on Learning Representations,
  ICLR 2018}, 2018.

\bibitem[GMOV18]{sewoong}
Weihao Gao, Ashok~Vardhan Makkuva, Sewoong Oh, and Pramod Viswanath.
\newblock Learning one-hidden-layer neural networks under general input
  distributions.
\newblock {\em CoRR}, abs/1810.04133, 2018.

\bibitem[JSA15]{janzamin2015beating}
Majid Janzamin, Hanie Sedghi, and Anima Anandkumar.
\newblock Beating the perils of non-convexity: Guaranteed training of neural
  networks using tensor methods.
\newblock {\em arXiv}, pages arXiv--1506, 2015.

\bibitem[LMZ20]{LiMZ20}
Yuanzhi Li, Tengyu Ma, and Hongyang~R. Zhang.
\newblock Learning over-parametrized two-layer neural networks beyond ntk.
\newblock In {\em Conference on Learning Theory 2020}, volume 125, pages
  2613--2682. PMLR, 2020.

\bibitem[LY17]{LiY17}
Yuanzhi Li and Yang Yuan.
\newblock Convergence analysis of two-layer neural networks with relu
  activation.
\newblock In {\em Advances in Neural Information Processing Systems 30}, pages
  597--607, 2017.

\bibitem[MR18]{manurangsi2018computational}
Pasin Manurangsi and Daniel Reichman.
\newblock The computational complexity of training relu (s).
\newblock {\em arXiv preprint arXiv:1810.04207}, 2018.

\bibitem[MSS16]{ma2016polynomial}
Tengyu Ma, Jonathan Shi, and David Steurer.
\newblock Polynomial-time tensor decompositions with sum-of-squares.
\newblock In {\em 2016 IEEE 57th Annual Symposium on Foundations of Computer
  Science (FOCS)}, pages 438--446. IEEE, 2016.

\bibitem[Sha18]{shamir2018distribution}
Ohad Shamir.
\newblock Distribution-specific hardness of learning neural networks.
\newblock {\em Journal of Machine Learning Research}, 19(32):1--29, 2018.

\bibitem[SJA16]{sedghi2016provable}
Hanie Sedghi, Majid Janzamin, and Anima Anandkumar.
\newblock Provable tensor methods for learning mixtures of generalized linear
  models.
\newblock In {\em Artificial Intelligence and Statistics}, pages 1223--1231.
  PMLR, 2016.

\bibitem[Sol17]{soltanolkotabi2017learning}
Mahdi Soltanolkotabi.
\newblock Learning relus via gradient descent.
\newblock In {\em Advances in neural information processing systems}, pages
  2007--2017, 2017.

\bibitem[SSSS17]{shalev2017failures}
Shai Shalev-Shwartz, Ohad Shamir, and Shaked Shammah.
\newblock Failures of gradient-based deep learning.
\newblock In {\em Proceedings of the 34th International Conference on Machine
  Learning-Volume 70}, pages 3067--3075, 2017.

\bibitem[VW19]{vempala2019gradient}
Santosh Vempala and John Wilmes.
\newblock Gradient descent for one-hidden-layer neural networks: Polynomial
  convergence and sq lower bounds.
\newblock In {\em COLT}, volume~99, 2019.

\bibitem[ZLJ16]{zhang2016l1}
Yuchen Zhang, Jason~D Lee, and Michael~I Jordan.
\newblock L1-regularized neural networks are improperly learnable in polynomial
  time.
\newblock In {\em 33rd International Conference on Machine Learning, ICML
  2016}, pages 1555--1563, 2016.

\bibitem[ZPS17]{zhangps17}
Qiuyi Zhang, Rina Panigrahy, and Sushant Sachdeva.
\newblock Electron-proton dynamics in deep learning.
\newblock {\em CoRR}, abs/1702.00458, 2017.

\bibitem[ZSJ{\etalchar{+}}17]{zhong2017recovery}
Kai Zhong, Zhao Song, Prateek Jain, Peter~L Bartlett, and Inderjit~S Dhillon.
\newblock Recovery guarantees for one-hidden-layer neural networks.
\newblock In {\em Proceedings of the 34th International Conference on Machine
  Learning-Volume 70}, pages 4140--4149, 2017.

\bibitem[ZYWG19]{zgu}
Xiao Zhang, Yaodong Yu, Lingxiao Wang, and Quanquan Gu.
\newblock Learning one-hidden-layer relu networks via gradient descent.
\newblock In {\em The 22nd International Conference on Artificial Intelligence
  and Statistics}, pages 1524--1534. PMLR, 2019.

\end{thebibliography}

\newpage
\appendix

\section{Deferred Preliminaries}

\subsection{Hermite polynomials}
\label{app:hermite}
Recall the definition of the probabilist's Hermite polynomials:
\[H_n(x) = (-1)^n e^{\frac{x^2}{2}}\cdot\frac{d^2}{dx^2}e^{-\frac{x^2}{2}}.\]
Under this definition, the first four Hermite polynomials are
\[H_0(x) = 1, ~~H_1(x) = x, ~~H_2(x) = x^2 - 1, ~~ H_3(x) = x^3-3x.\]
The Hermite polynomials comprise an orthogonal basis of the Hilbert space $\mL^2(\bR, \omega)$
where $\omega$ is the standard Gaussian measure on $\R$.
In this function space, the inner product is defined as $\langle f,g \rangle = \bE_{x\sim\mN(0,1)} [f(x) g(x) ]$.
Under this inner product, we have
\[\langle H_m, H_n\rangle = \bE_{x\sim\mN(0,1)} H_m(x)H_n(x) = m!\cdot \mathbb{I}[m=n].\]

The \emph{normalized} probabilist's Hermite polynomials are given by $\wh{H}_\ell(x) \triangleq \frac{1}{\sqrt{\ell!}} H_\ell(x)$; these comprise an orthonormal basis of $\mathcal{L}^2(\R,\omega)$. Finally, given $x\in\R^d$, we define the \emph{Hermite tensor} $S_\ell(x) \in (\R^d)^{\otimes \ell}$ to be the tensor whose $(i_1,\ldots,i_\ell)$-th entry is given as follows. Suppose that element $j\in[d]$ appears among $i_1,\ldots,i_\ell$ a total of $n_j$ times. Then the $(i_1,\ldots,i_\ell)$-th entry of $S_\ell$ is given by $\prod^d_{j=1} \wh{H}_{n_j}(x_j)$.

\subsection{Measuring distance between networks}
Here, we will show that parameter closeness implies closeness in the square loss, which also implies the closeness between moment tensors.

\begin{lemma}[Lemma 3.3 from \cite{chen2021efficiently}]\label{lem:closerelus}
    For any unit vectors $u,u'$,
    \begin{equation}
        \mathbb{E}\left[\relu(\langle u, x\rangle) - \relu(\langle u', x\rangle)\right]^2 \le \frac{5}{6}\|u-u'\|_2^2.
    \end{equation}
\end{lemma}

\begin{lemma}\label{lem:paramtoL2}
For two 2-layer ReLU networks $\flamu, f_{\vec{\lambda}',\vec{u}'}$ with the same width $k$ and for which 
$\norm{\vlam}_1, \norm{\vec{\lambda}'}_1 \le \radius $ for some $\radius > 0$, we have
\[\|f_{\lambda, \vec{u}}-f_{\lambda', \vec{u}'}\|_2\lesim k\max(1,\radius )\cdot \paramdist((\vlam,\vec{u}), (\vec{\lambda}',\vec{u'})).\]
\end{lemma}

\begin{proof}
By AM-GM, for any permutation $\pi$:
\begin{equation*}
\label{eqn:lemma1-1}
\begin{aligned}
&~~~~~\|f_{\lambda, \vec{u}}-f_{\lambda', \vec{u}'}\|_2^2 = \mathbb{E}\left(\sum_{i=1}^k \lambda_i \cdot \relu (\langle u_i, x\rangle) - \lambda_{\pi(i)}' \cdot \relu (\langle u_{\pi(i)}', x\rangle)\right)^2\\
&= \mathbb{E}\left(\sum_{i=1}^k (\lambda_i - \lambda_{\pi(i)}') \cdot \relu (\langle u_i, x\rangle) + \sum_{i=1}^k \lambda_{\pi(i)}' \cdot \left(\relu(\langle u_i, x\rangle) - \relu (\langle u_{\pi(i)}', x\rangle)\right)\right)^2\\
&\leqslant (2k)\cdot \sum_{i=1}^k \mathbb{E}\left[ (\lambda_i - \lambda_{\pi(i)}')^2 \cdot \relu (\langle u_i, x\rangle)^2\right] + (2k)\cdot \sum_{i=1}^k\mathbb{E}\left[\lambda_i'^2\cdot\left(\relu(\langle u_i, x\rangle) - \relu (\langle u_{\pi(i)}', x\rangle)\right)^2\right]\\
&\leqslant k\sum_{i=1}^k (\lambda_i-\lambda_{\pi(i)}')^2 + \frac{5k}{3}\sum_{i=1}^k \lambda_{\pi(i)}'^2\cdot \|u_i-u_{\pi(i)}'\|_2^2\\
&\leqslant k^2 \cdot\paramdist((\vlam,\vec{u}), (\vec{\lambda'},\vec{u'}))^2 + \frac{5k^2}{3} \radius ^2 \cdot \paramdist((\vlam,\vec{u}), (\vec{\lambda'},\vec{u'}))^2 \\
&\leqslant k^2(1+2\radius ^2)\cdot \paramdist((\vlam,\vec{u}), (\vec{\lambda'},\vec{u'}))^2,
\end{aligned}
\end{equation*}
where in the third step we used Lemma~\ref{lem:closerelus}. Therefore, we conclude that
\[\|f_{\lambda, \vec{u}}-f_{\lambda', \vec{u}'}\|_2\leqslant k\sqrt{1+2R^2}\cdot \paramdist((\vlam,\vec{u}), (\vec{\lambda'},\vec{u'})),\]
as claimed.
\end{proof}

\begin{lemma} \label{lem:2normtofrob} 
    For the 2-norm of $\flamu$ and the Frobenius norm of its moment tensor, it holds that:
    \[\norm{T_\ell(\vlam,\vec{u})}_F\le \sqrt{2\pi}\cdot d^{\ell/2} \cdot \norm{\flamu}_2.\]
\end{lemma}

\begin{proof}
To prove this equation, we firstly need to decompose the ReLU activation function $\relu(\cdot)$ with the Hermite polynomials. From Lemma A.2 of \cite{goel2020superpolynomial}, we have:
\[\relu(x) = \sum_{k=0}^{\infty} c_k H_k(x),\]
where the coefficients are
\[c_0 = \frac{1}{\sqrt{2\pi}}, c_1 = \frac12, ~c_{2k} = \frac{(-1)^{k+1}(2k-3)!!}{\sqrt{2\pi}\cdot (2k)!}~(k\geqslant 1) ~~\text{and}~~c_{2k+1} = 0.\]
Note that, these coefficients are slightly different from \cite{goel2020superpolynomial} since our definition of $H_n$ is unnormalized (which means their $H_n$ stands for our $\wh{H}_n$).

~\\
Therefore, we can express $T_{\ell}(\lambda, \vec{u})$ as:
\[\mathbb{E}\left[f_{\lambda, \vec{u}}(x)\cdot H_{\ell}(x)\right] = \ell !\cdot c_{\ell }\sum_{i=1}^k \lambda_i u_i^{\otimes \ell} = \ell !\cdot c_{\ell} T_{\ell}(\lambda, \vec{u}).\]
Now, we take the Frobenius norm of both sides, we have:
\begin{equation*}
\begin{aligned}
\ell !\cdot \frac{(\ell-3)!!}{\sqrt{2\pi} \ell!}\cdot \|T_{\ell}(\lambda, \vec{u})\|_F &= \left\|\mathbb{E}\left[f_{\lambda, \vec{u}}(x)\cdot H_{\ell}(x)\right]\right\|_F = \sqrt{\sum_{\alpha\in [d]^l}\left(\mathbb{E}\left[f_{\lambda, \vec{u}}(x)\cdot H_{\ell, \alpha}(x)\right]\right)^2}\\
&\leqslant \sqrt{\sum_{\alpha\in [d]^l} \mathbb{E}[f_{\lambda, \vec{u}}(x)^2]\cdot \mathbb{E}\left[ H_{\ell,\alpha}(x)^2\right]} = \|f_{\lambda, \vec{u}}\|_2 \cdot \sqrt{\mathbb{E} \|H_\ell(x)\|_F^2}\\
&< \|f_{\lambda, \vec{u}}\|_2 \cdot \sqrt{d^\ell \cdot \ell!}
\end{aligned}
\end{equation*}
Since $(\ell-3)!! < \sqrt{\ell !}$,  we can conclude that:
\[\|T_{\ell}(\lambda, \vec{u})\|_F \leqslant \sqrt{2\pi}\cdot d^{\ell/2} \|f_{\lambda, \vec{u}}\|_2.\]
\end{proof}

\begin{lemma}\label{lem:contract_close}
For any unit vector $g\in\mathbb{S}^{d-1}$, it holds that:
\[\norm{M^g_\ell(\vlam,\vec{u}) - M^g_\ell(\vec{\lambda'},\vec{u'})}_{\op}\le \norm{T_\ell(\vlam,\vec{u}) - T_\ell(\vec{\lambda'},\vec{u'})}_F.\]
\end{lemma}

\begin{proof}
Notice that:
\[M_{\ell}^g(\lambda, \vec{u}) = T_{\ell}(\lambda, \vec{u})[g,\ldots, g, :, :], ~~M_{\ell}^g(\lambda', \vec{u}') = T_{\ell}(\lambda', \vec{u}')[g,\ldots, g, :, :].\]
Then, for any unit vector $v\in\mathbb{S}^{d-1}$, we have:
\begin{equation*}
\begin{aligned}
&~~~~~v^{\top}\left(M_{\ell}^g(\lambda, \vec{u}) - M_{\ell}^g(\lambda', \vec{u}')\right)v = \left(T_{\ell}(\lambda, \vec{u}) - T_{\ell}(\lambda', \vec{u}')\right)[g,\ldots, g, v, v]\\
& = \langle T_{\ell}(\lambda, \vec{u}) - T_{\ell}(\lambda', \vec{u}'),  g\otimes \ldots \otimes g\otimes v\otimes v\rangle \\
& \leqslant \left\|T_{\ell}(\lambda, \vec{u}) - T_{\ell}(\lambda', \vec{u}'\right\|_F \cdot \left\|g\otimes \ldots \otimes g\otimes v\otimes v\right\|_F \\
& = \left\|T_{\ell}(\lambda, \vec{u}) - T_{\ell}(\lambda', \vec{u}'\right\|_F \cdot \left(\|g\|_2^{l-2} \|v\|_2^2\right)\leqslant\left\|T_{\ell}(\lambda, \vec{u}) - T_{\ell}(\lambda', \vec{u}')\right\|_F,
\end{aligned}
\end{equation*}
which leads to the conclusion that
\[\norm{M^g_\ell(\vlam,\vec{u}) - M^g_\ell(\vec{\lambda'},\vec{u'})}_{\op}\le \norm{T_\ell(\vlam,\vec{u}) - T_\ell(\vec{\lambda'},\vec{u'})}_F.\]
\end{proof}

\subsection{Proofs for anti-concentration}
\label{app:anti}

We use the following standard bound:

\begin{lemma}\label{lem:vec_to_proj}
    Given unit vector $u\in\S^{d-1}$, if $g$ is a random unit vector, then with probability at least $1 - \delta$,
    \begin{equation}
        \sqrt{d}|\iprod{u,g}| \in [c \delta,c'\sqrt{\ln(2/\delta)}]
    \end{equation}
    for some absolute constants $c,c' > 0$.
\end{lemma}

\begin{proof}
    We can write $g$ as $h/\norm{h}$ for $h\sim\calN(0,\Id)$. Then $\iprod{u,h} \sim\calN(0,\norm{v})$ and we have 
    \begin{align}
        \Pr{|\iprod{u, h}| &\ge \sqrt{2\ln(2/\delta)}} \le \delta/2 \\
        \Pr{|\iprod{u ,h}| &\le (\delta/2)/\sqrt{2/\pi}} \le \delta/2\,.
    \end{align}
    Additionally, $\Pr{|\norm{h} - \sqrt{d}| \ge \sqrt{d}/2} \le \exp(-\Omega(d))$. The lemma follows by a union bound.
\end{proof}

\noindent We can now complete the proofs of the two lemmas from Section~\ref{sec:anti}:

\begin{proof}[Proof of Lemma~\ref{lem:anti}]
    Take any $u \triangleq \frac{1}{\norm{u_i + \sigma\cdot u_j}}\cdot (u_i + \sigma\cdot u_j)$ and note that by Lemma~\ref{lem:vec_to_proj} applied to $\delta = 1/10k^2$, we have
    $|\iprod{u,g}| \in \bigl[\frac{c}{10k^2\sqrt{d}}, \frac{c'\sqrt{\ln(20k^2)}}{\sqrt{d}}\bigr]$. The lemma follows by a union bound over all $i,j\in[k]$ and $\sigma\in\brc{\pm 1}$. 
\end{proof}

\begin{proof}[Proof of Lemma~\ref{lem:vlbd}]
    Take $u$ in Lemma~\ref{lem:vec_to_proj} to be $u_i$ and $\delta = 1/10k$ to conclude that $\sqrt{d}|\iprod{u,g}| \ge c/10k$. The lemma follows by a union bound over $i$.
\end{proof}

\noindent Finally, we record the following elementary inequality:

\begin{lemma}\label{lem:squaretosigneddiff}
    For any $a,b \in\R$ satisfying $|a|, |b| \le 1$,
    \begin{equation}
         \min_{\sigma\in\brc{\pm 1}} |a - \sigma b|^2 \le |a^2 - b^2| \le 2\min_{\sigma\in\brc{\pm 1}} |a - \sigma b|\,.
    \end{equation}
\end{lemma}

\begin{proof}
    Both bounds are immediate from the fact that $|a^2 - b^2| = |a - b|\cdot |a + b|$.
\end{proof}

\section{Deferred Proofs From Section~\ref{sec:main}}

\subsection{Proof of Lemma~\ref{lem:moment_est}}
\label{app:momentest_defer}
    This was essentially, e.g., in Corollary 42 of~\cite{smallcovers}. Note that while that work considered the case of positive $\lambda_i$, their proof of Corollary 42 does not use positivity. Additionally, their guarantee is stated for all even $\ell$ only because their algorithm only makes use of even $\ell$, even though their proof of Corollary 42 applies equally well to $\ell = 1$.

    Additionally, their guarantee is stated in terms of $\relu$ activation instead of absolute value activation. But note that $|z| = \relu(z) + \relu(-z)$, so because $c_\ell$ is the $\ell$-th normalized probabilist's Hermite coefficient of $\relu(\cdot)$ and satisfies $c_\ell = \Theta(\ell^{-5/2})$ (see e.g. \cite[Lemma A.2]{goel2020superpolynomial}). we conclude that the corresponding Hermite coefficient for $|\cdot|$ is $2c_\ell$, so $\E{\wh{T}} = T_\ell$. The claim for $\ell = 1$ follows similarly.

\subsection{Proof of Lemma \ref{lem:estimate_residual_moments}}
\label{app:residual_defer}
    The degree-1 Hermite coefficients of $f_{0,\vhlam,\vu}$ is zero, while the degree-1 Hermite coefficients of $\fwlamu$ are given by $w$, so the expectation of $\frac{1}{N}\sum^N_{a = 1}(y_a - f_{0,\vlam,\vu}(x_a))x_a$ is $w$. By Lemma~\ref{lem:moment_est}, the deviation between the empirical mean and the population mean is bounded by $\xi'/2$ provided $N \ge \poly(d)\,(\radius^2 + \norm{\vhlam}^2_1)/\xi'^2$. This establishes the first bound.
    
    Note that by Lemmas~\ref{lem:paramtoL2}, \ref{lem:2normtofrob}, and \ref{lem:contract_close}, we have the following
    \begin{align}
        \norm{M^g_\ell(\vlam, \vu) - M^g_\ell(\vhlam, \vhu) - M^g_\ell(\vlam_{[\kres]}, \vu_{[\kres]})}_2 &\le \xi' \label{eq:Trefclose2}  
    \end{align}
    for all $\ell = 2,4,\ldots,2k + 2$. 
    By Lemma~\ref{lem:moment_est}, the deviation between the empirical mean $\frac{1}{N}\sum^N_{a=1} (y_a - f_{\vhlam, \vhu}(x_a)) x_a$ and the population mean is bounded by $\xi'/2$ provided $N \ge d^{O(\ell)}\,(\radius^2 + \norm{\vhlam}^2_1)/\xi'^2$. This establishes the second bound by triangle inequality with \eqref{eq:Trefclose2}.

\subsection{Proof of Lemma \ref{lem:twoneuron}}
\label{app:defer_twoneuron}
    By AM-GM, it suffices to show
    \begin{equation}
        \norm{\fwlamu - f_{0,\vhlam,\vhu} - h}  \lesim \poly(d,\radius) (\epsilon' + \xi')/\omega + \omega \label{eq:wantamgm}
    \end{equation} holds for any $0 < \omega < 1$.
    By hypothesis, for all $i,j\in[\kres]$ we have that $|v_i - v_j| \le \epsilon'$, so by the fact that we are conditioning on the event of Lemma~\ref{lem:anti}, 
    \begin{equation}
        \norm{u_i - u_j} \lesim \epsilon' k^2\sqrt{d}\,. \label{eq:uiujallclose}
    \end{equation}
    Let $S^+$ and $S^-$ denote the partition of $[\kres]$ into indices $i$ for which $\argmin_{\sigma\in\brc{\pm 1}}\norm{u_i - \sigma \cdot u_1}$ is $+1$ and $-1$ respectively. Also define $\lambda^+ \triangleq \sum_{i\in S^+} \lambda_i$ and $\lambda^-\triangleq \sum_{i\in S^-} \lambda_i$. Consider the network 
    \begin{equation}
        h^* \triangleq \lambda^+ \relu(\iprod{u_1,x}) + \lambda^- \relu(\iprod{-u_1,x})\,.
    \end{equation}
    Then by Lemma~\ref{lem:paramtoL2}, 
    \begin{equation}
        \norm{h^* - \fres} \lesim \epsilon' \radius\, k^2\sqrt{d}\,. \label{eq:hstar_to_fres}
    \end{equation}
    By Lemma~\ref{lem:2normtofrob} applied to $h^* - \fres$, we conclude that 
    \begin{align}
        \norm{T_1((\lambda^+,\lambda^-),(u_1,-u_1)) - T_1(\vlam_{[\kres]}, \vu_{[\kres]})}_2 &\lesim  \epsilon' \radius\, k^2d^{3/2} \\
        \norm{T_2((\lambda^+,\lambda^-),(u_1,-u_1)) - T_2(\vlam_{[\kres]}, \vu_{[\kres]})}_F &\lesim  \epsilon'\radius\, k^2d^2 \epsilon'\,.
    \end{align}
    By combining these with Lemma~\ref{lem:estimate_residual_moments}, we find that the empirical estimates $\frac{1}{N}\sum^N_{a=1} 2(y_a - f_{\vhlam,\vhu}(x_a))x_a$ and $\frac{1}{N}\sum^N_{a=1} \sqrt{2\pi} (y_a - f_{\vhlam,\vhu}(x_a)) (x_ax_a^\top - \Id)$ are $(\xi' + \poly(d)\radius\epsilon')$-close to 
    \[
    T_1((\lambda^+,\lambda^-),(u_1,-u_1)) = (\lambda^+ - \lambda^-)u \text{ and }T_2((\lambda^+,\lambda^-),(u_1,-u_1)) = (\lambda^+ + \lambda^-)u_1 u_1^\top,
    \]
    provided $N \ge \poly(d)\,(\radius^2 + \norm{\vhlam}^2_1)/\xi'^2$.
    By right-multiplying the latter empirical estimate by $g$, we get an estimate of $T_2((\lambda^+, \lambda^-), (u_1,-u_1))\, g = \iprod{u_1,g}\cdot (\lambda^+ + \lambda^-)\cdot u$ whose error in $L_2$ is of the same order. By the fact that we are conditioning on the event of Lemma~\ref{lem:vlbd}, $|\iprod{u_1,g}| \gtrsim 1/k\sqrt{d}$.

    We conclude that we have access to both $(\lambda^+ + \lambda^-) u_1$ and $(\lambda^+ - \lambda^-) u_1$, and thus to $\lambda^+ u_1$ and $\lambda^- u_1$ and also the scalars $\lambda^+$ and $\lambda^-$, to error of order $\poly(d) (\radius\epsilon' + \xi')$. If $\lambda^+$ and $\lambda^-$ are both at most $c\omega$ in magnitude for sufficiently small $c$, then $\norm{h^*} \lesim \omega$ and the estimator $h \equiv 0$ already achieves the desired bound in \eqref{eq:wantamgm}. Otherwise, we can use our estimates of $\lambda^+ u_1$ and $\lambda^- u_1$ to estimate $u_1$ to $L_2$ error of order $\poly(d)(\radius\epsilon' + \xi)/\omega$. By Lemma~\ref{lem:paramtoL2}, we obtain an estimate $h = \mu^+ \relu(\iprod{u,\cdot}) + \mu^-\relu(\iprod{-u,\cdot})$ satisfying 
    \begin{equation}
        \norm{h^* - h} \le \poly(d,\radius) (\epsilon' + \xi')/\omega\,. \label{eq:h2h}
    \end{equation}
    Combining Eqs.~\eqref{eq:l2close}, \eqref{eq:hstar_to_fres}, \eqref{eq:h2h} yields the desired bound \eqref{eq:wantamgm}
    upon noting that the bound on $\norm{h^* - h}$ dominates among the three bounds.

\subsection{Proof of Lemma \ref{lemma:validation}}
\label{app:validate}

Since in a one-hidden-layer ReLU network (without bias term), $F(0) = 0$, and furthermore $F$ is a $2k\radius $-Lipschitz continuous function, we conclude that $|F(x)|\leqslant 2k\radius \cdot \|x\|_2$ for $\forall x\in \mathbb{R}^d$. Next, we can apply the proof of Lemma A.1 of \cite{chen2021efficiently} to show that $G(x) := F(x)^2 - \mu$ is a zero-centered, sub-exponential random variable with sub-exponential norm $\|G\|_{\Psi_1} = \mathcal O(\mu + 4\radius ^2 k^3)$. Finally, by using the concentration property of sub-exponential random variables, we conclude that:
\[\left|\frac{1}{N}\sum_{i=1}^N G(x_i)\right|\leqslant t \]
with probability at least $1-\delta$. Here, the sample size $N = \Theta(K^2\log(1/\delta) / t^2)$, where $K = \mu + 4\radius ^2 k^3$.

\end{document}